\documentclass[letterpaper]{article} 
\usepackage[]{aaai24}  
\usepackage{times}  
\usepackage{helvet}  
\usepackage{courier}  
\usepackage[hyphens]{url}  
\usepackage{graphicx} 
\usepackage{natbib}  
\usepackage{caption} 
\usepackage{algorithm}
\usepackage{algorithmic}

\usepackage{newfloat}
\usepackage{listings}
\DeclareCaptionStyle{ruled}{labelfont=normalfont,labelsep=colon,strut=off} 
\floatstyle{ruled}
\newfloat{listing}{tb}{lst}{}
\floatname{listing}{Listing}
\usepackage{caption, subcaption, amsmath, amssymb,amsthm}
\usepackage{multicol}
\newtheorem{thm}{Theorem}
\newtheorem{lem}[thm]{Lemma}

\newtheorem{thm*}{Theorem}
\newtheorem*{theorem*}{Theorem}
\newtheorem{prop*}{Proposition}
\newtheorem{lem*}{Lemma}
\newtheorem{exmp*}{Example}

\renewcommand{\>}{{\rightarrow}}

\newcommand{\R}{{\mathbb R}}

\newcommand{\1}{{\mathbf 1}}

\renewcommand{\b}{{\mathbf b}}

\newcommand{\h}{{\mathbf{h}}}

\newcommand{\w}{{\mathbf w}}
\newcommand{\x}{{\mathbf x}}
\newcommand{\bu}{{\mathbf u}}

\newcommand{\bphi}{{\boldsymbol \phi}}
\newcommand{\boldeta}{{\boldsymbol \eta}}

\title{Half-Space Feature Learning in Neural Networks}
\author {
    Yadav Mahesh Lorik\textsuperscript{\rm 1},
    Harish Guruprasad Ramaswamy\textsuperscript{\rm 1},
    Chandrashekar Lakshminarayanan\textsuperscript{\rm 1}
}
\affiliations {
    \textsuperscript{\rm 1}IIT Madras \\
    cs20s010@cse.iitm.ac.in, hariguru@cse.iitm.ac.in, 
    chandrashekar@cse.iitm.ac.in
}

\begin{document}

\maketitle

\title{Half-Space Feature Learning in Neural Networks}



\begin{abstract}
There currently exist two extreme viewpoints for neural network feature learning -- (i) Neural networks simply implement a kernel method (a la NTK) and hence no features are learned (ii) Neural networks can represent (and hence learn) intricate hierarchical features suitable for the data. We argue in this paper neither interpretation is likely to be correct based on a novel viewpoint. Neural networks can be viewed as a mixture of experts, where each expert corresponds to a (number of layers length) path through a sequence of hidden units. We use this alternate interpretation to motivate a model, called the Deep Linearly Gated Network (DLGN), which sits midway between deep linear networks and ReLU networks. Unlike deep linear networks, the DLGN is capable of learning non-linear features (which are then linearly combined), and unlike ReLU networks these features are ultimately simple -- each feature is effectively an indicator function for a region compactly described as an intersection of (number of layers) half-spaces in the input space. This viewpoint allows for a comprehensive global visualization of features, unlike the local visualizations for neurons based on saliency/activation/gradient maps. Feature learning in DLGNs is shown to happen and the mechanism with which this happens is through learning half-spaces in the input space that contain smooth regions of the target function. Due to the structure of DLGNs, the neurons in later layers are fundamentally the same as those in earlier layers -- they all represent a half-space -- however, the dynamics of gradient descent impart a distinct clustering to the later layer neurons. We hypothesize that ReLU networks also have similar feature learning behavior.

\end{abstract}

\section{Introduction}

The theoretical analysis of the success behind neural network learning has been a topic of huge interest. The questions remain largely unanswered, but several schools of thought have emerged, out of which we highlight the two most prominent ones below.

The theory of Neural Tangent Kernels \cite{Jacot+18} explains the infinite width neural network (under an appropriate initialization) with great precision by showing that they are equivalent to kernel machines. Some papers have taken this to the conclusion that finite width neural networks trained under gradient descent are essentially approximations to a kernel machine (e.g. \cite{Arora+19}). The Theorems in these papers require an extremely large number of hidden neurons per layer and are not directly applicable to real-world neural networks. Empirically, it has been shown that the kernel machines under-perform the SGD trained finite neural nets on most architectures/datasets.

The other prominent line of thought makes the case that the hierarchical nature of neural networks enables the representation of intricate highly symmetric functions with much lesser parameters than interpolation-based algorithms like kernel methods. There have been several examples of functions that can be represented efficiently with large-depth models, that need exponentially many parameters with any shallow model \cite{Telgarsky+16, Shalev-Shwartz+17}. It has also been observed that the number of linear regions in a neural network can be exponential in the number of neurons and depth of the network, \cite{Montufar+14}, and this representational capability has been proposed as a reason behind the success of neural network learning.

Both these views come with their own set of issues. The NTK viewpoint is a technically elegant argument, but it still falls short in explaining the superior performance of finite-width methods under SGD. It also essentially implies that no features are being learned by the neural network during training. The hierarchical feature learning viewpoint also fails by not showing even a single example where a neural network \emph{learns} automatically the existing hierarchical features in the data. In fact, it has been shown \cite{Shalev-Shwartz+17} that no gradient-based method is capable of learning certain multi-layer hierarchical structures in the data. A proper understanding of this feature learning interpretation, if true, would enable the construction of simple synthetic datasets with a hierarchical structure for which neural networks outperform kernel methods. Such a construction would have value in bringing out the strengths of the deep neural network architecture over shallow kernel methods, but no such dataset exists to the best of our knowledge. All known works that give examples of distributions where kernel methods perform provably poorly compared to gradient descent on a neural network are essentially those where the target function $f^*(\x) =  u(w^\top \x)$ for some univariate function $u$ and vector $\w$ \cite{Damian22, Shi22,Daniely20}. While these form an interesting set of problems, these do not sufficiently explain the `features' learned by the neural network in general learning problems.

In this paper, we provide supporting evidence for a position that straddles the two extremes --  
\emph{Neural networks learn data-appropriate non-linear features in the early stages of SGD training while the loss is still high, and combine these features linearly to learn a low-loss model in the later stages}. Other similar lines of work show that the NTK `aligns' with the data for deep linear networks \cite{Atanasov+22}, demonstrate this occurring in ReLU nets via analysing the spectral dynamics of the empirical NTK \cite{Baratin+21}, and support this by an analysis of the loss landscape \cite{Fort+20}.

The inference that the neural tangent kernel (and hence the learned features) change during training to become better suited to the task at hand is well established. But the mechanism by which this change happens is still largely murky. We provide a framework by means of which the dynamics of the learned features can be better studied. 

\subsection{Overview}
In Section \ref{sec:MoE} we frame ReLU networks and some other neural network models as a mixture of experts model with a large number of simple experts. In Section \ref{sec:active-path-regions} we build the main tools required for the analysis of deep networks in the mixture of experts paradigm. In Sections \ref{sec:neural-net-feature-learning} and \ref{sec:GD-resource} we demonstrate the feature learning dynamics of deep networks on some simple synthetic datasets and make some pertinent observations.

\section{The Mixture of Simple Experts Framework}
\label{sec:MoE}

\begin{figure*}
    \centering
    \includegraphics[scale = .9, width=0.45 \textwidth]{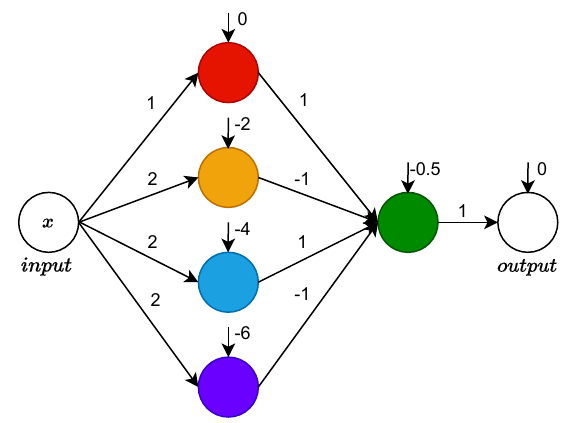}
    \includegraphics[scale = .9, width=0.45 \textwidth]{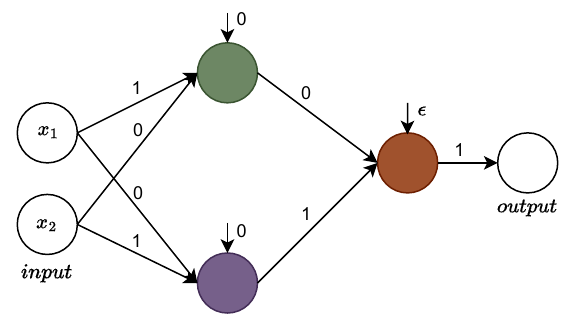}
    \caption{Example ReLU nets with weights and biases such that-- (\textbf{Left}:) the active path region through red and green hidden nodes is $[0.5,1.5] \cup [2.5, 3.5]$. (\textbf{Right}): The active path region through the green and brown hidden nodes abruptly changes from the right half of the $2$-dimensional input space to only the first quadrant when the parameter $\epsilon$ changes sign.}
    \label{fig:ReLU_nets_active_paths}
\end{figure*}
The classic mixture of experts (MoE) model \cite{Jacobs+91} has been a mainstay of machine learning with its intuitive breakdown of the prediction model into multiple `experts' which are chosen by a `gating model'. The classic MoE paradigm typically sets each individual expert to be a neural network and learns them together. In this section, we flip the argument on its head, and reframe a single neural network as a mixture of experts.

Consider the following simple mixture of experts model with set of experts $\Pi$.

\begin{equation}
\widehat y(\x)= \sum_{\pi \in \Pi} f_\pi(\x) g_\pi(\x) 
\label{eqn:MoE}
\end{equation}
where $f_\pi(\x) \in \{0,1\}$ is the `hard' gating model for expert $\pi$ and identifies the `region of expertise' for expert $\pi$. A `soft' gating model where $f_\pi(\x) \in [0,1]$ can also be used. $g_\pi(\x)$, corresponds to the prediction of expert $\pi$ for input $\x$. In the mixture of `simple' experts paradigm, the model $g_\pi$ is extremely simple, (in our examples it will be either a constant or a linear function of $\x$).  

Now, consider a simple neural network model that takes in $\x\in \R^d$ as input (called layer $0$) and has $L-1$ hidden layers (called layers $1$ to $L-1$) consisting of $m$-nodes each and single output node (called layer $L$). We identify the set of experts $\Pi$ in the MoE model with the set of paths between input and output in the neural network. As each path simply corresponds to a sequence of $L-1$ hidden nodes, the cardinality of $\Pi$ is $m^{L-1}$. Individually parameterising $f_\pi$ and $g_\pi$ for each path $\pi$ is practically impossible, and some clever choices are required to get a workable model whose size does not explode exponentially with the number of layers. Various such choices give rise to recognisable models.

\subsection{The ReLU Network as a Mixture of Simple Experts}
Each node $i$ in layer $\ell \in [L]=\{1, 2,\ldots,L-1,L\}$ is associated with a weight vector $\w_{\ell,i}$. Nodes in layer $\ell=1$ have weight vectors $\w_{\ell,i} \in \R^d$, and other layers have weight vectors $\w_{\ell,i} \in \R^m$. For each layer $\ell$, the weights $\w_{\ell,i}$ form the rows of the matrix $W_\ell$. Let $\pi=(i_1, i_2, \ldots, i_{L-1})\in  [m]^{L-1}$ represent a full path giving a sequence of $L-1$ hidden nodes. Consider the following gating and expert model.

\begin{equation}
f_\pi(\x) = \prod_{\ell=1}^{L-1} \1\left(\w_{\ell,i_\ell}^\top \h_{\ell-1}(\x) \geq 0 \right) \label{eqn:ReLU-net-gating-net}
\end{equation}

where $\h_0(\x) = \x$ and $\h_\ell(\x) = \bphi(W_\ell \h_{\ell-1}(\x))$, with $\bphi$ representing the point-wise ReLU function. Now consider the following expert prediction model which is a linear function of the input $\x$.
\begin{equation}
g_\pi(\x) = \left[ \prod_{\ell=2}^{L-1} w_{\ell,i_\ell,i_{\ell-1}}\right] w_{L,1,i_{L-1}} \w_{1,i_1}^\top \x
\label{eqn:RelU-net-individual-expert-model}
\end{equation}

\begin{thm}
    Let $\widehat y(\x)$ be the output of the mixture of experts model in Equation \ref{eqn:MoE}, with the gating model $f_\pi$ and expert model $g_\pi$ given by Equations \ref{eqn:ReLU-net-gating-net} and \ref{eqn:RelU-net-individual-expert-model}. Then 
    \[
        \widehat y(\x) = W_L \h_{L-1}(\x)
    \]
    where $h_0(\x)=\x$ and $h_\ell(\x)=\bphi(W_\ell h_{\ell-1}(\x))$ for $\ell \in \{1,\ldots,L-1\}$.
    \label{thm:relu-theorem}
\end{thm}
We use networks without bias for simplicity, but biases can easily be incorporated by adding one extra neuron per layer representing the constant $1$ and freezing the input weights of these constant nodes.
The above decomposition of ReLU nets as a combination of  simple linear model experts has independent value. 

The gating model itself has a ReLU function in its definition and is thus hard to interpret or analyse. This gating model can represent complex shapes, e.g. there exist parameters $W_1, \ldots, W_L$ such that $A_\pi = \{\x: f_\pi(\x) = 1\}$ is not even a connected set for some path $\pi$. See  Figure \ref{fig:ReLU_nets_active_paths}(a) for an example. The gating model is discontinuous in its parameters $W,\b$ as well, i.e. there exists two parameters $W$ and $W'$ that are arbitrarily close, but for some path $\pi$ the set $A_\pi$ is drastically different for these two parameter settings. See Figure \ref{fig:ReLU_nets_active_paths}(b) for an example.
These issues with individual gating models makes both the learning dynamics and the learned model hard to interpret (at least via the MoE paradigm).

\subsection{Deep Linearly Gated Network}

Taking inspiration from the MoE view of the ReLU network, we propose the following gating and expert model. Let $\pi=(i_1, i_2, \ldots, i_{L-1})\in  [m]^{L-1}$ represent a full path giving a sequence of $L-1$ hidden nodes.

\begin{equation}
f_\pi(\x) = \prod_{\ell=1}^{L-1} \1\left(\boldeta_{\ell, i_\ell}(\x) \geq 0 \right) \label{eqn:DLGN-gating-model}
\end{equation}
where $\boldeta_0(\x) = \x$ and $\boldeta_\ell(\x) = W_\ell \boldeta_{\ell-1}(\x)$.
\begin{equation}
g_\pi(\x) = \left[ \prod_{\ell=2}^{L-1} u_{\ell,i_\ell,i_{\ell-1}}\right] u_{L,1,i_{L-1}} \bu_{1,i_1}^\top \x
\label{eqn:DLGN-individual-expert-model}
\end{equation}
The two differences between the above model and the ReLU net in Equations \ref{eqn:ReLU-net-gating-net} and \ref{eqn:RelU-net-individual-expert-model} is that the gating model no longer uses a ReLU non-linearity, and the individual expert model uses a distinct set of parameters $U_1, \ldots, U_L$ with the same shape as the $W$ parameters.

We call this architecture as the deep linearly gated network (DLGN)(illustration figures to be found in the appendix). Similar to the ReLU network MoE, it can also be computed efficiently without explicit enumeration of all the paths $\Pi$. 
\begin{thm}
Let $\widehat y(\x)$ be the output of the mixture of experts model in Equation \ref{eqn:MoE}, with the gating model $f_\pi$ and expert model $g_\pi$ given by Equations \ref{eqn:DLGN-gating-model} and \ref{eqn:DLGN-individual-expert-model}. Then 
\[
        \widehat y(\x) = U_L \left(\h_{L-1}(\x) \right)
\]
where $\boldeta_0(\x)=\h_0(\x) = \x$ and $\boldeta_\ell(\x)= W_\ell \boldeta_{\ell-1}(\x)$  and $\h_\ell(\x) = \1(\boldeta_\ell(\x) \geq 0) \circ \left( U_\ell \h_{\ell-1}(\x) \right) $ for $\ell \in \{1,\ldots,L-1\}$.
\label{thm:dlgn-theorem}
\end{thm}

The main advantage of the DLGN over the ReLU network lies in its much simpler gating network (when viewed in the MoE paradigm). It is always guaranteed that for any path $\pi$, its active region given by $A_\pi = \{\x \in \R^d : f_\pi(\x) = 1\}$ is an intersection of $L-1$ half-spaces, and hence a convex polyhedron. 


\subsection{Piecewise Constant DLGN}
As the expert model $g_\pi$ is linear in $\x$, and $f_\pi(\x)$ is equal to an indicator function for a convex polyhedron in $\R^d$, the total model $\widehat y(\x)$ is clearly still piece-wise linear. It is possible to simplify the model even further by replacing the linear function $g_\pi$ with a constant independent of $\x$. A simple way of accomplishing this is by setting $\h_0(\x)$ as any constant non-zero vector instead of $\x$. We set $\h_0$ as the all-ones vector as a convention. This has the same effect as replacing $\x$ in the RHS of Equation \ref{eqn:DLGN-individual-expert-model} by the all-ones vector. This variant is called the piece-wise-constant (PWC) DLGN  and its empirical performance on most datasets (in terms of test accuracy) is indistinguishable from the piece-wise-linear DLGN.  In all our experiments to illustrate the DLGN, we use DLGN-PWC due to its extra simplicity.


\begin{table*}[ht]
    \centering
\begin{tabular}{|c|c|c|c|c|}
\hline
& Gating model $f_\pi (\x)$ & Expert prediction model $g_\pi(\x)$  & Train Accuracy & Test Accuracy  \\ \hline
ReLU network    
& $\prod_{\ell=1}^{L-1} \1\left(\w_{\ell,i_\ell}^\top \h_{\ell-1}(\x)  \geq 0 \right)$ 
& $\left[ \prod_{\ell=2}^{L-1} w_{\ell,i_\ell,i_{\ell-1}}\right] w_{L,1,i_{L-1}} \w_{1,i_1}^\top \x$ 
& 96.3
& 72.17 \\
DLGN 
& $\prod_{\ell=1}^{L-1} \1\left( \boldeta_{\ell, i_\ell}(\x)  \geq 0 \right)$ 
& $\left[ \prod_{\ell=2}^{L-1} u_{\ell,i_\ell,i_{\ell-1}}\right] u_{L,1,i_{L-1}} \bu_{1,i_1}^\top \x$
& 85.3
& 73.2 \\
DLGN-PWC 
& $\prod_{\ell=1}^{L-1} \1\left( \boldeta_{\ell, i_\ell}(\x)  \geq 0 \right)$ 
& $\left[ \prod_{\ell=2}^{L-1} u_{\ell,i_\ell,i_{\ell-1}}\right] u_{L,1,i_{L-1}} \bu_{1,i_1}^\top \1$ 
& 84.9
& 72.2 \\
DLN
& $1$
& $\left[ \prod_{\ell=2}^{L-1} w_{\ell,i_\ell,i_{\ell-1}}\right] w_{L,1,i_{L-1}} \w_{1,i_1}^\top \x$ 
& 42.3
&40.79 \\ 
\hline
\end{tabular}

\caption{Mixture of simple expert formulations. Gating model $f_\pi$ and expert prediction model $g_\pi$ for a path $(i_1,i_2, \ldots, i_{L-1})$, on a deep network. The ReLU network and DLN have parameters $W_1, W_2, \ldots, W_L$. The DLGN have parameters $U_1,\ldots, U_L$ in addition to the $W$ parameters. The layer outputs for a ReLU and deep linear network are given by $\h_\ell(\x)$ and $\boldeta_\ell(\x)$. i.e. $\h_0(\x) = \boldeta_0(\x) = \x$ and $\h_\ell(\x) = \bphi(W_l \h_{\ell-1}(\x))$ and  $\boldeta_\ell(\x) = W_\ell \boldeta_{l-1}(\x)$. The train and test accuracy are shown on CIFAR-10 dataset with the architectures explained in detail in the appendix.}
\label{tab:MoE_formulations}
\end{table*}

A significant focus of this paper revolves around positioning DLGNs as a pivotal bridge between ReLU nets and deep linear networks, with a key emphasis on their interpretability. To further that claim, we need some empirical evidence that DLGNs can perform comparably to ReLU nets while being superior to static feature methods. Indeed, we find that in a simple experiment on CIFAR10 (refer Table \ref{tab:MoE_formulations}) with a simple 5 layer convolutional architecture, ReLU nets, DLGN and DLGN-PWC perform similarly, and outperform the linear model with empirical NTK features and deep linear networks by a wide margin. Our experimental results on CIFAR 10 and CIFAR 100 with the ResNet architecture demonstrates that performance of DLGNs is much better than linear models, and slightly worse than ReLU networks(comparison results to be found in the appendix).

\section{Active Path Regions and the Overlap Kernel}
\label{sec:active-path-regions}


In this section we build the tools required to study feature learning that happens in neural networks.

\subsection{Active Path Regions vs Activation Pattern Regions}

For a deep network defined using the MoE paradigm in Equation \ref{eqn:MoE}, we define two complementary objects. The set of paths $A(\x)$ active for a given input $\x$, and the set of inputs $A_\pi$ active for a given path $\pi$.

\begin{align*}
    A(\x) &= \{\pi \in \Pi : f_\pi(\x) = 1 \} \\
    A_\pi &= \{\x \in \R^d : f_\pi(\x) = 1 \}
\end{align*}

A related term that has been in use is the activation pattern region $A_\b$ for any given bit pattern $\b \in \{0,1\}^{(L-1)m}$. The Activation pattern region for ReLU nets is defined as follows: 
\begin{align*}
A_\b 
&= \{ \x \in \R^d : h_{\ell, i} (\x) > 0 , \forall \ell,i  \text{ s.t.} b_{\ell,i}=1,  \\
& ~~~~~~~~~ \text{ and } h_{\ell, i} (\x) \leq 0, \forall \ell \in [L-1],i \in [m] \text{ s.t.} b_{\ell,i}=0 \}
\end{align*} 
where $\h_\ell$ forms the activation output of layer $\ell$ as defined in Equation \ref{eqn:ReLU-net-gating-net}. Thus $A_\b$ corresponds to the subset of the input space for the activation pattern of the neurons matches $\b$ exactly. A similar definition is possible with DLGN as well, where $\h_\ell$ is replaced with $\boldeta_\ell$ as defined in Equation \ref{eqn:DLGN-gating-model}. 

These activation pattern regions have been the focus of several papers that identify that the activation patterns partition the input space into convex polyhedra, and that the ReLU net is linear inside each of these regions \cite{Hanin+19a, Hanin+19b}. The active path regions $A_\pi$ can also be viewed as a union of several (precisely $2^{(m-1)*(L-1)}$) activation pattern regions $A_\b$. 

The number of activation pattern regions for ReLU nets can be exponential in the number of neurons. Most results in representation theory of neural networks typically  \cite{Telgarsky+16, Montufar+14} exploit this exponential blow-up of activation pattern regions to construct intricate highly symmetric and fast varying functions that cannot be represented by shallow networks. However, it has been proven that a randomly initialised neural net has only polynomially many activation regions, and this number also has been observed empirically to not change drastically during training \cite{Hanin+19b}. This questions the utility of studying activation pattern regions in practical learning problems. 



We argue that active path regions are a better tool to  study feature learning than activation pattern regions because they are lesser in number ( $m^L$ vs $2^{mL}$) and have simpler structures (DLGN active path regions have $L$ faces as opposed to activation pattern regions having $mL$ faces). While active path regions are suited for studying both ReLU nets and DLGNs, they are especially apt for DLGNs because active path regions in DLGNs correspond to convex polyhedra.

\subsection{The Overlap Kernel}

An important consequence of the mixture of experts reformulation of neural nets is that we get a natural feature vector for a given data point: the set of paths active for a given data point. This leads to a natural kernel that we call the overlap kernel $\Lambda$ (a la \cite{LakshmiSingh20}). 


\begin{equation*}
\Lambda(\x, \x') = |A(\x) \cap A(\x')| = \sum_{\pi \in \Pi} f_\pi(\x) f_\pi(\x') 
\end{equation*}

Due to the Cartesian product structure of the set of paths the above expression can be simplified easily. For the case of ReLU nets this reduces to
\[
\Lambda(\x, \x') = \prod_{l=1}^{L-1}  \sum_{i=1}^m \left(\1(h_{\ell,i}(\x) > 0)\cdot \1(h_{\ell,i}(\x') > 0) \right).
\]

Consider the learning problem where the simple experts $g_\pi$ are all restricted to constant values, and can be independently optimised. The learning problem for $g_\pi$ with fixed gating functions is equivalent to learning a linear model with inner product kernel given by the overlap kernel. Hence, studying the overlap kernel at different points during training gives valuable insights on the features being learnt by the model.

\begin{figure*}[ht]
    \begin{minipage}{0.3\textwidth}
    \includegraphics[width=\linewidth]{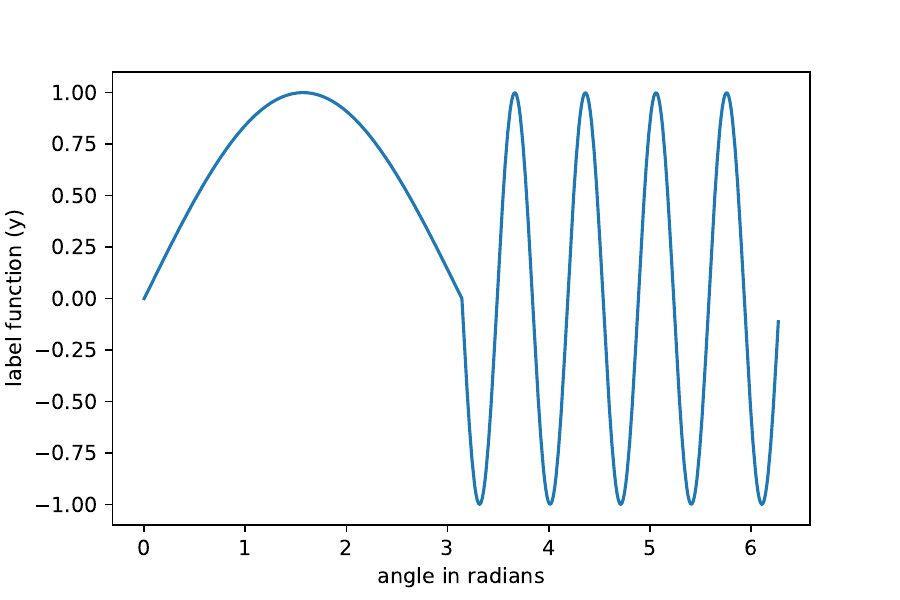}\par 
    \includegraphics[width=\linewidth]{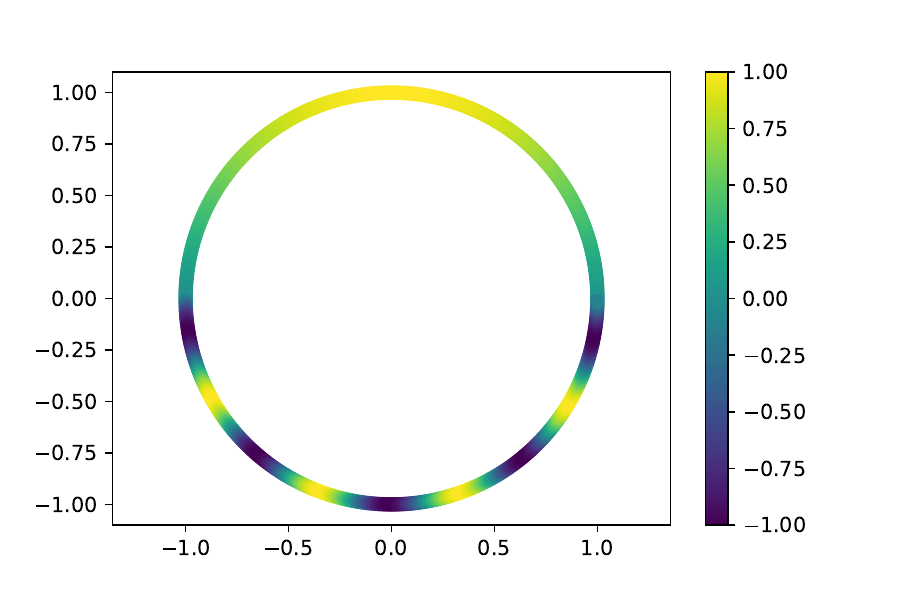}\par    
    \end{minipage}
    \begin{minipage}{0.69\textwidth}
     \includegraphics[width=\linewidth]{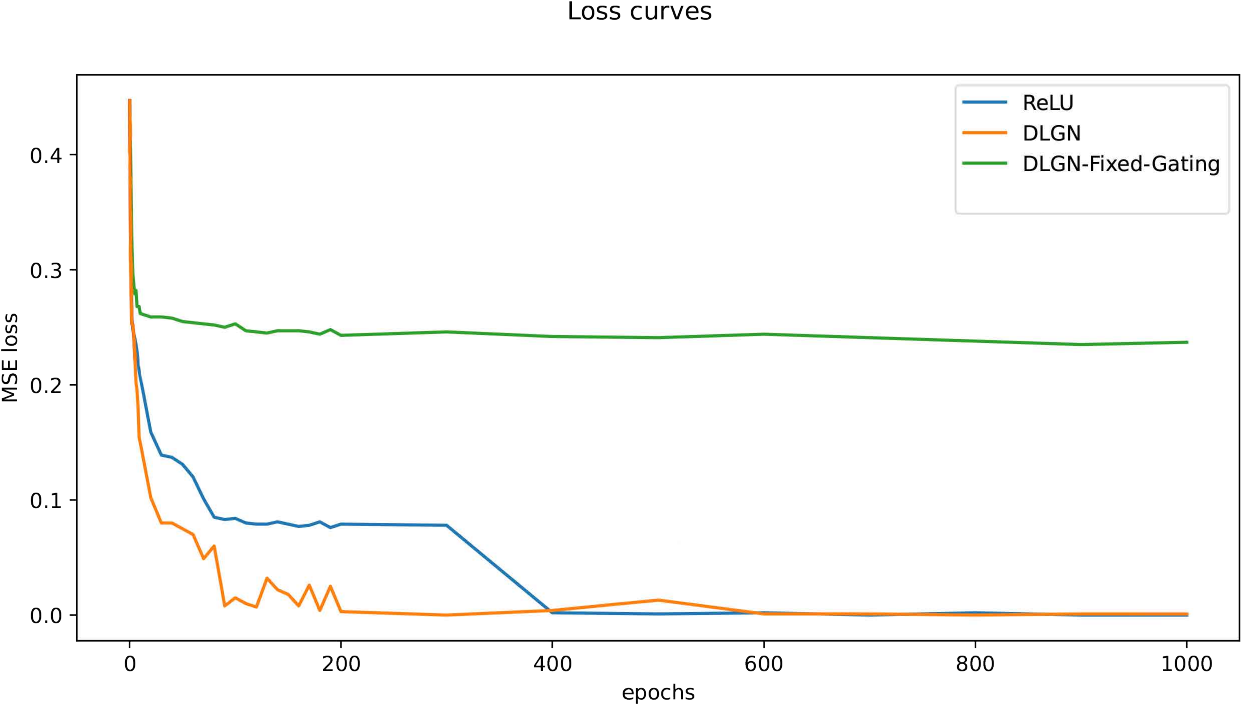}\par
\end{minipage}
    \caption{Left top: Target label $y$ as a function of the angles that input $\x$ makes. Left Bottom: A scatter plot of the data with points colored according to the target $y$. Right: The training loss of ReLU and DLGN models. DLGN models with fixed gating take more than 2000 iterations to converge to loss (~0.01) and are cropped here.}
    \label{fig:circle_dataset_summary}
\end{figure*}
\section{Neural Network Feature Learning}
\label{sec:neural-net-feature-learning}

In this section we use the tools built in Sections \ref{sec:MoE} and \ref{sec:active-path-regions} to illustrate the mechanism of feature learning in Neural networks which will help shed some light on interpretability of neural networks. Most of the results we illustrate here are on a simple 2-dimensional regression dataset, where the instances $\x$ are restricted to lie on a unit circle and the true labelling function $y(\x)$ varies with high frequency on one half of the circle and with lower frequency on the other half. Figure \ref{fig:circle_dataset_summary} gives an illustration of the dataset, and the squared loss of all the algorithms trained on the dataset as a function of the epochs.


The algorithms analysed are a simple 5-layer ReLU net (with 16 neurons per layer), a DLGN, and a DLGN with a fixed gating model $f_\pi$ that is randomly initialised. It is immediately obvious that the DLGN with a fixed gating model is incapable of learning features and has a difficult time optimising the training objective. We illustrate the learned features of the ReLU net and DLGN at different points during training via the overlap kernel.

\begin{figure*}
    \centering
    \includegraphics[width=0.15\textwidth]{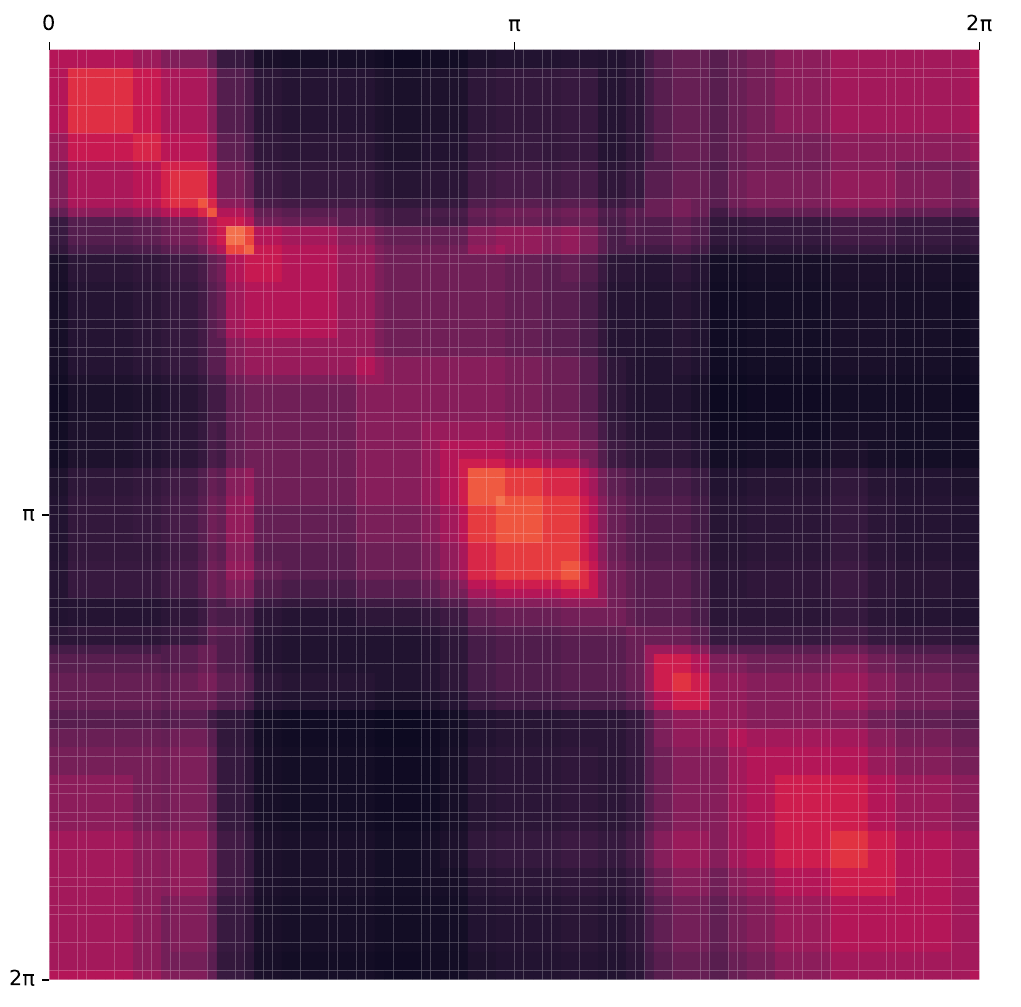}
    \includegraphics[width=0.15\textwidth]{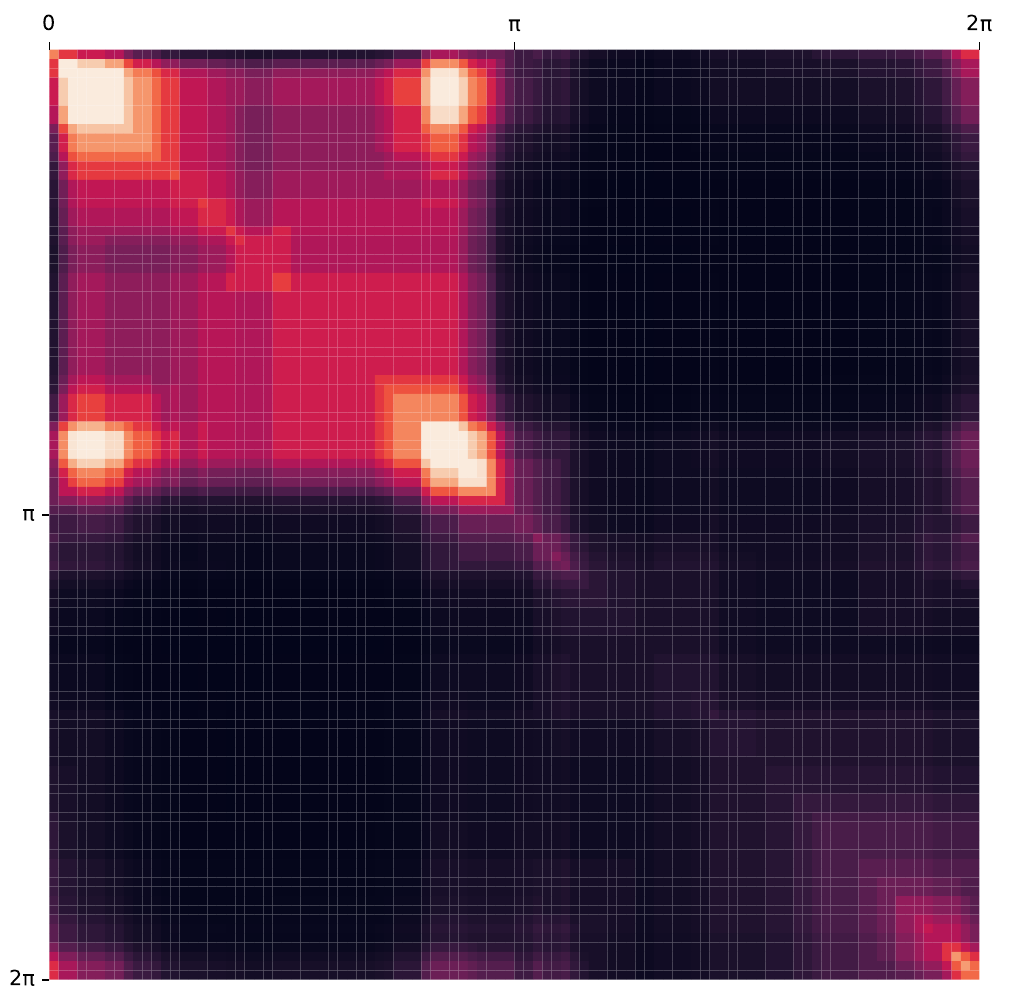}
    \includegraphics[width=0.18\textwidth]{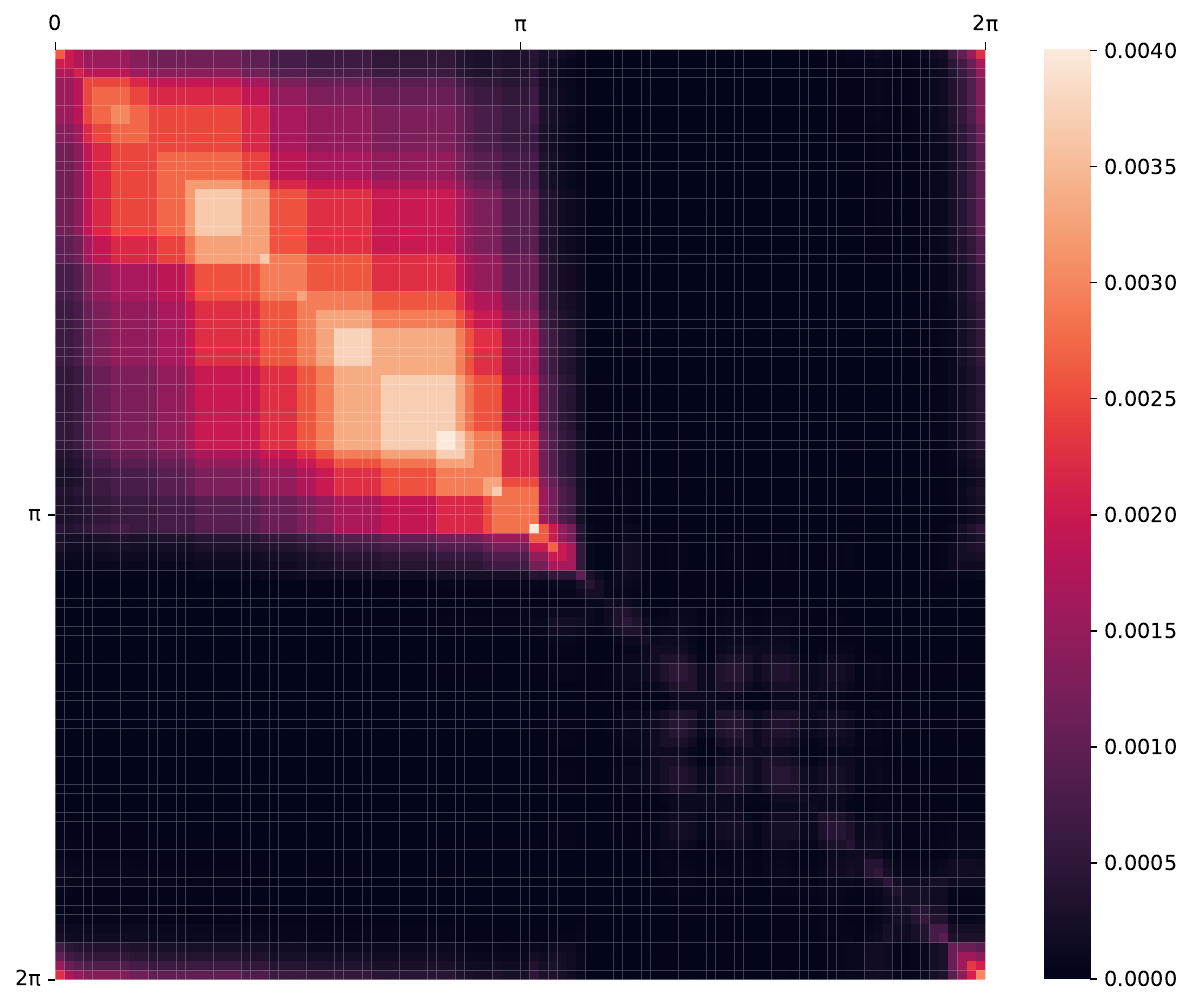}
    \includegraphics[width=0.15\textwidth]{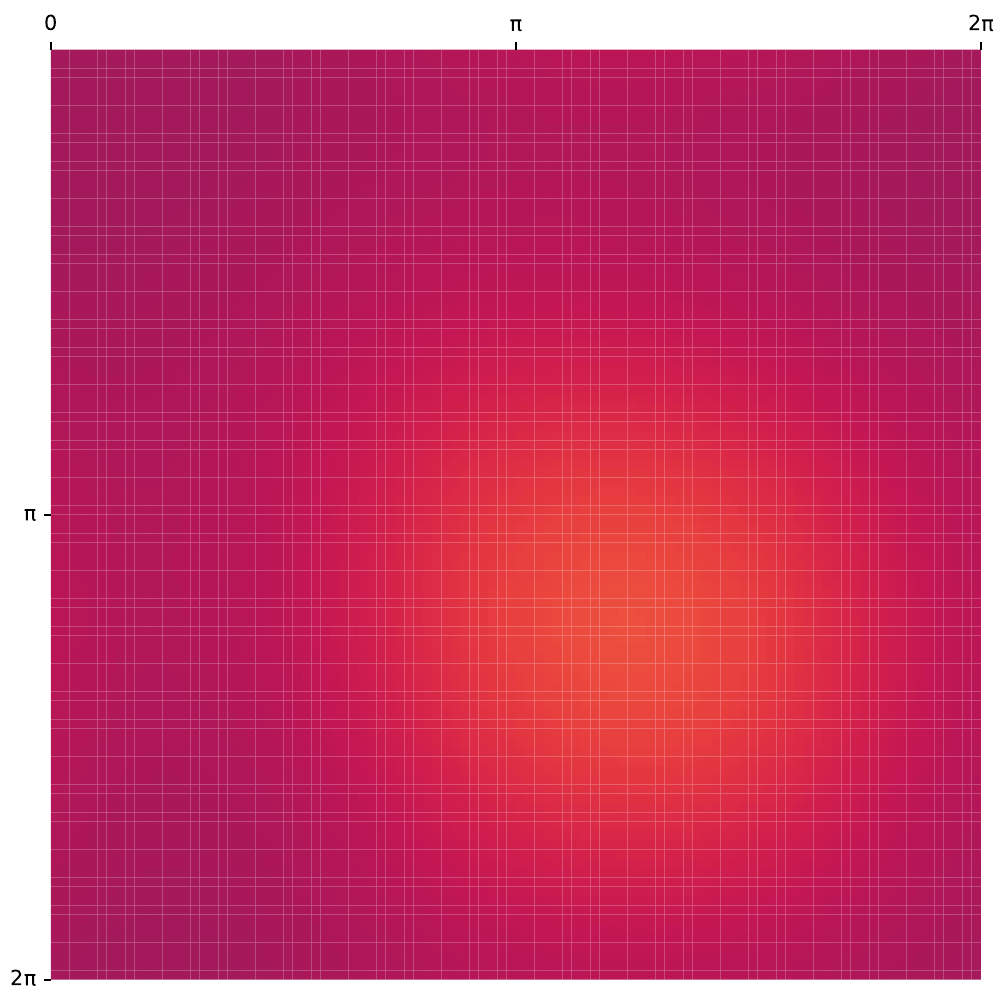}
    \includegraphics[width=0.15\textwidth]{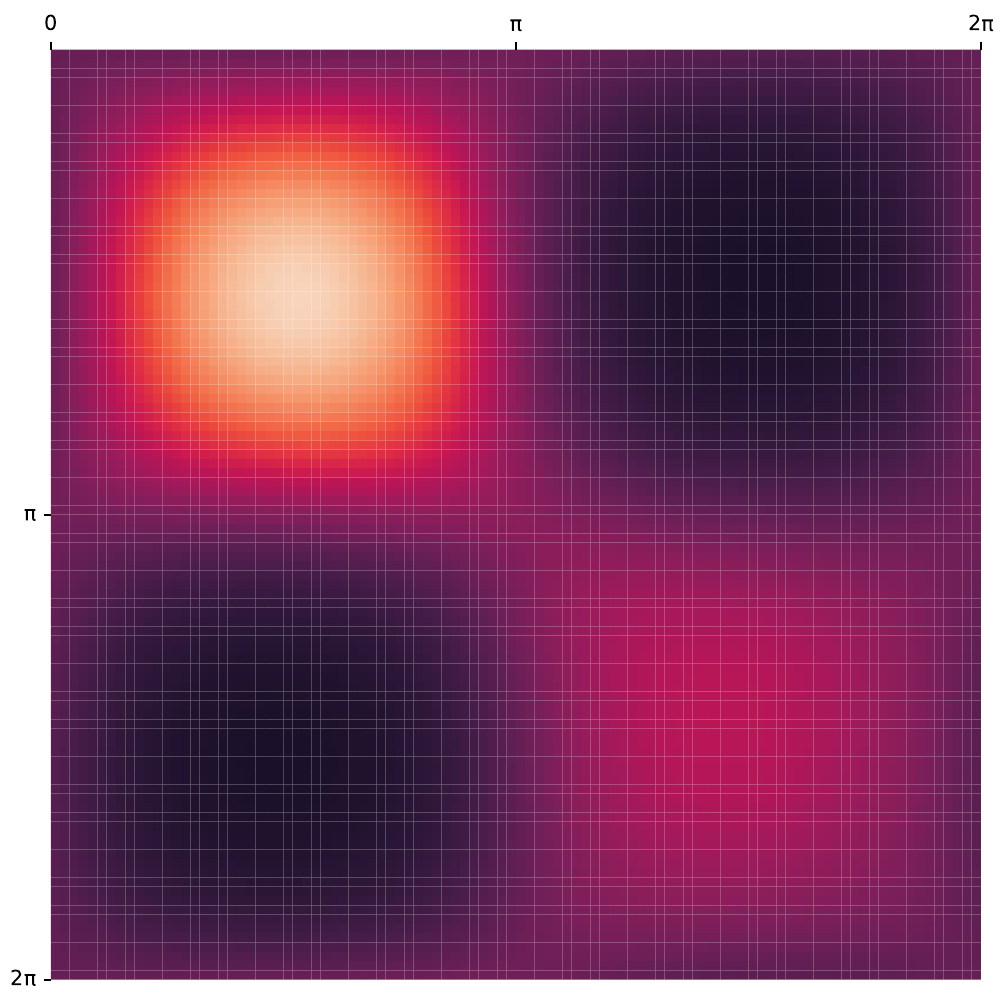}
    \includegraphics[width=0.18\textwidth]{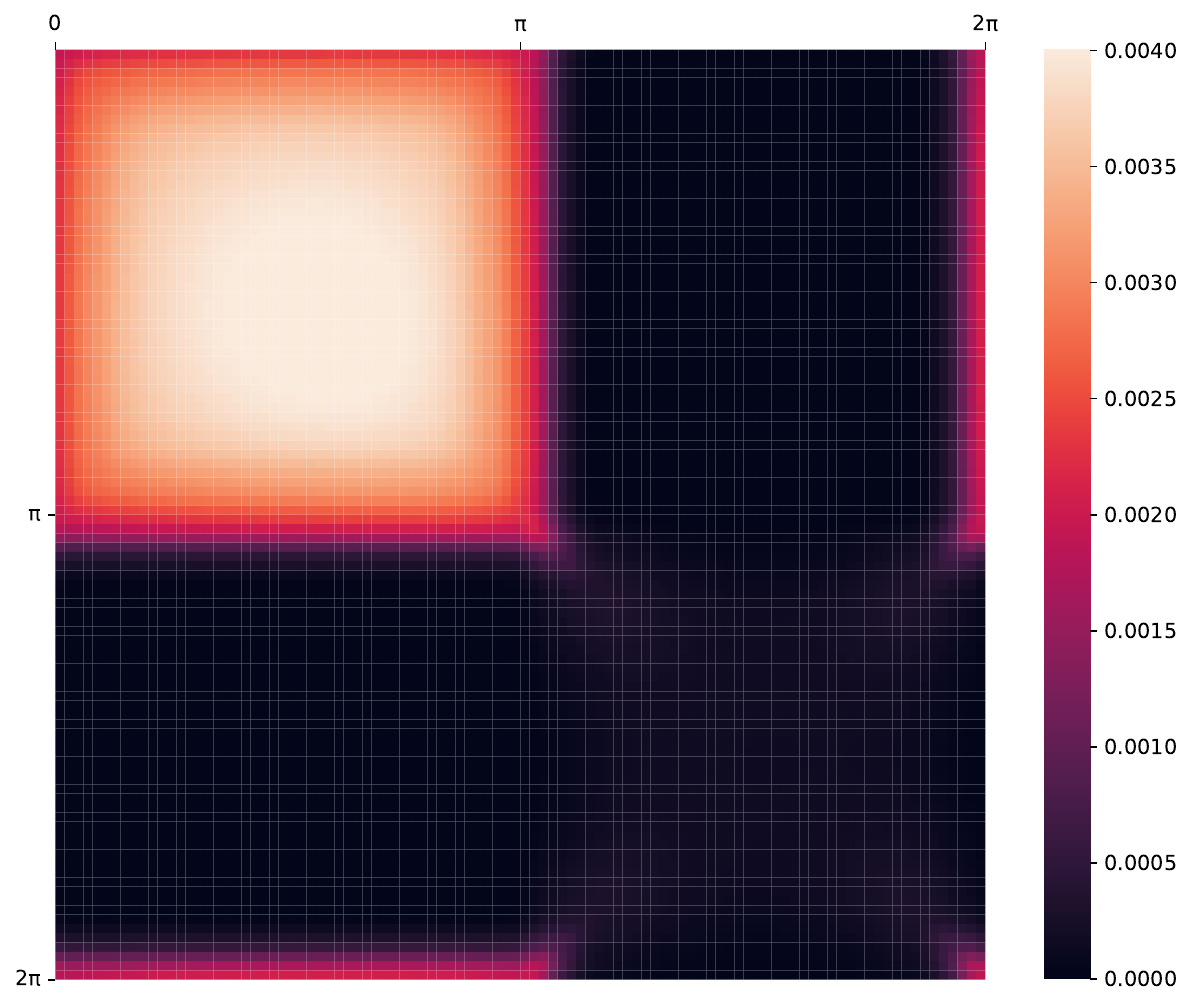}
    \caption{The (trace-normalised) overlap kernel for the ReLU net (left three images) and DLGN (right three images) at initialization (left), epoch 3 (middle) and epoch 200 (right). The data points are ordered based on angle -- first half corresponds to data in top half of the circle.}
    \label{fig:overlap_kernel_ReLU_DLGN}
\end{figure*}

\begin{figure*}
    \centering
    \includegraphics[width=0.14\textwidth]{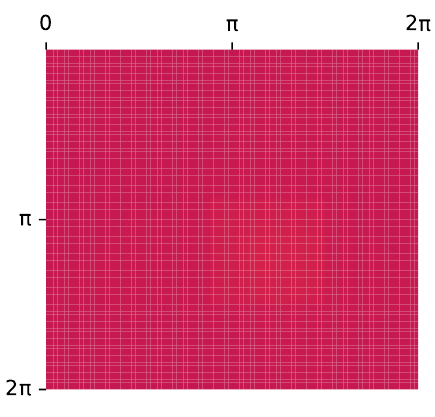}
    \includegraphics[width=0.14\textwidth]{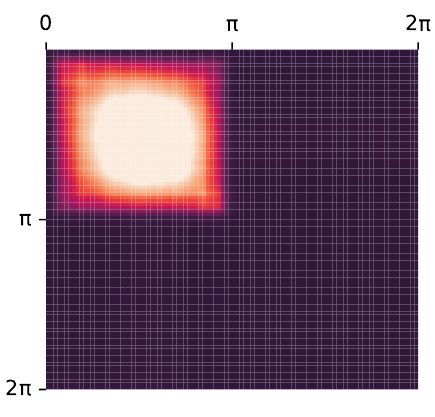}
    \includegraphics[width=0.19\textwidth]{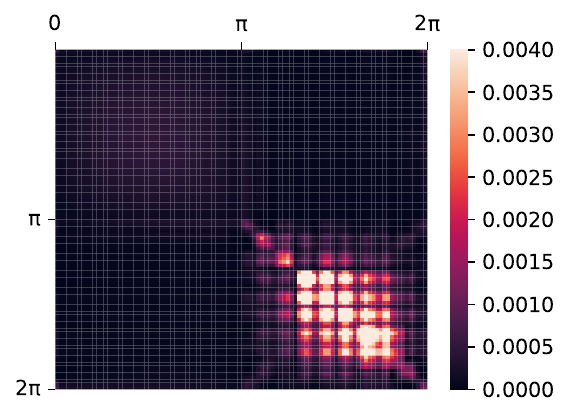}
    \includegraphics[width=0.14\textwidth]{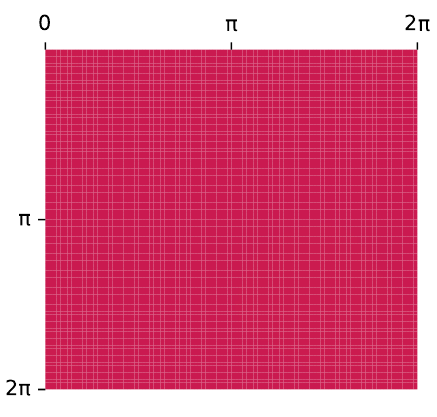}
    \includegraphics[width=0.14\textwidth]{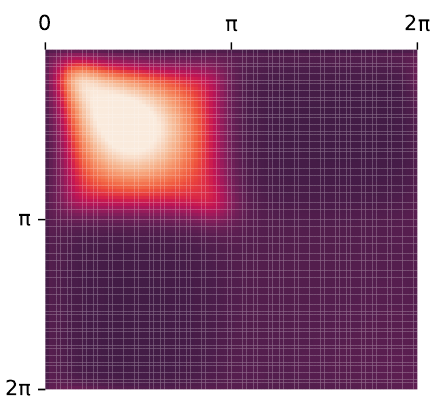}
    \includegraphics[width=0.19\textwidth]{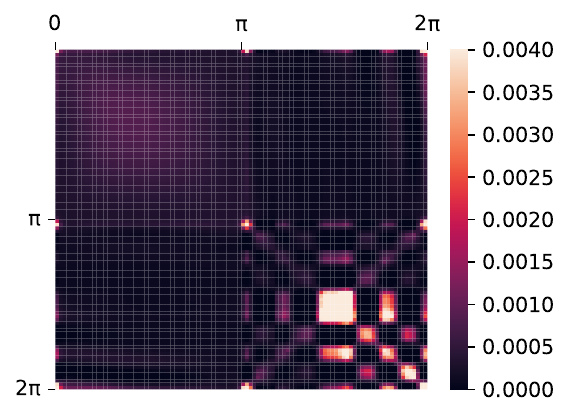}    
    \caption{The (trace-normalised) empirical Neural Tangent Kernel for the ReLU net (left three images) and DLGN (right three images) visualised as a matrix at initialization (left), epoch 3 (middle) and epoch 200 (right)}
    \label{fig:NTK_ReLU_DLGN}
\end{figure*}
\begin{figure*}
    \centering
    \includegraphics[width=0.32\textwidth]{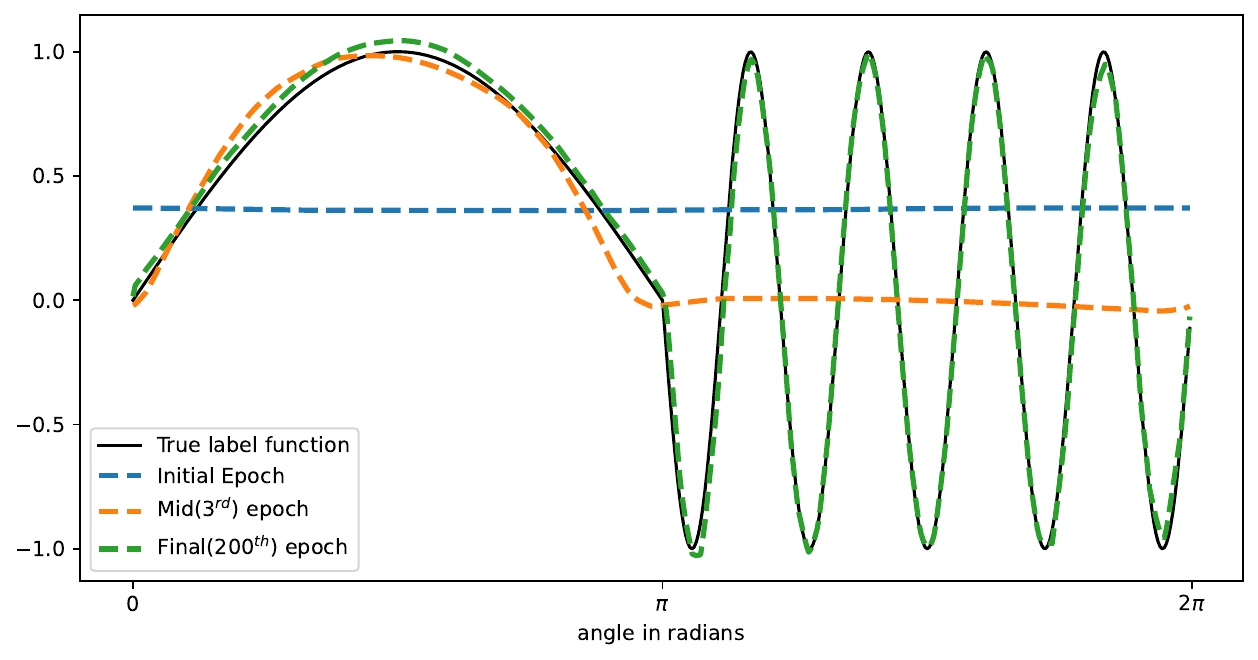}
    \includegraphics[width=0.32\textwidth]{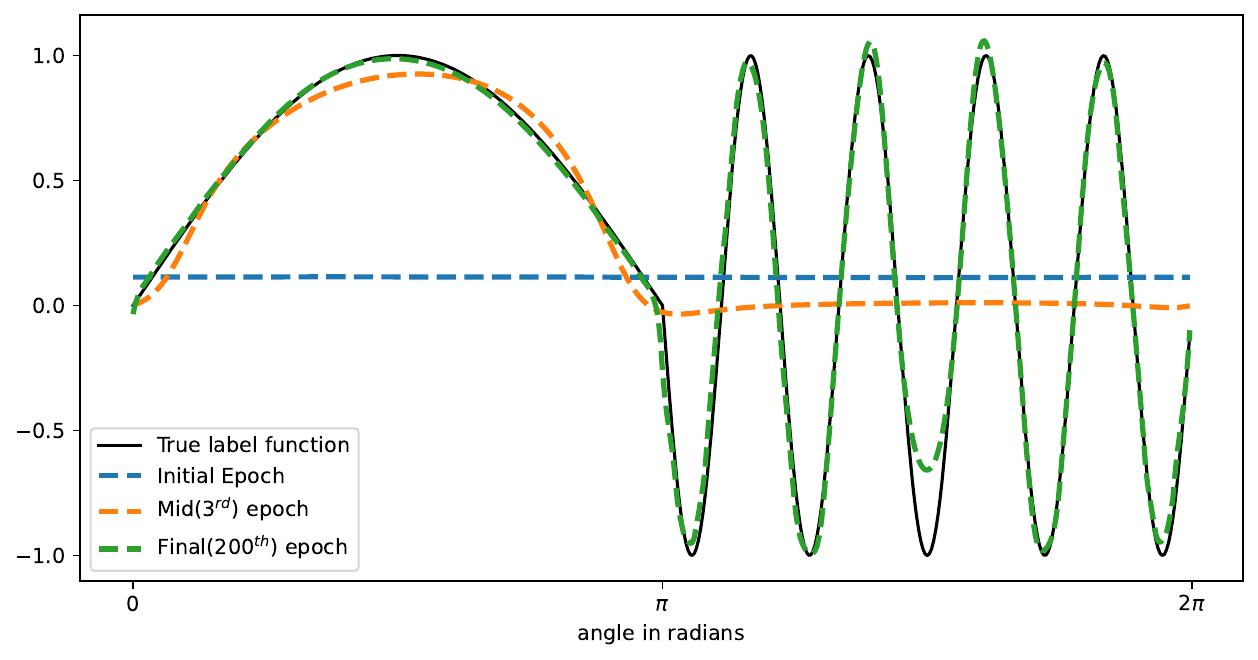}
    \includegraphics[width=0.32\textwidth]{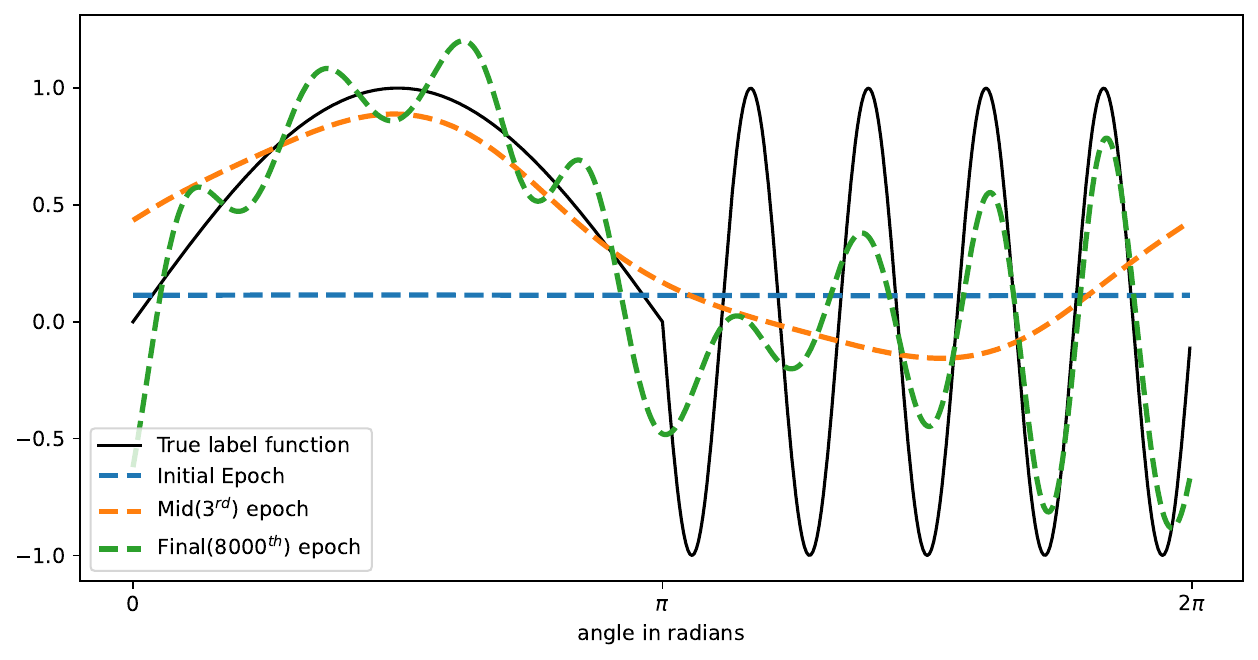}
    \caption{The model outputs of the ReLU net (left), DLGN (middle) and DLGN with the gating model $f_\pi$ frozen (right)  at 3 different epochs. }
    \label{fig:preds}
\end{figure*}
\begin{figure*}
    \centering
    \includegraphics[width=\textwidth]{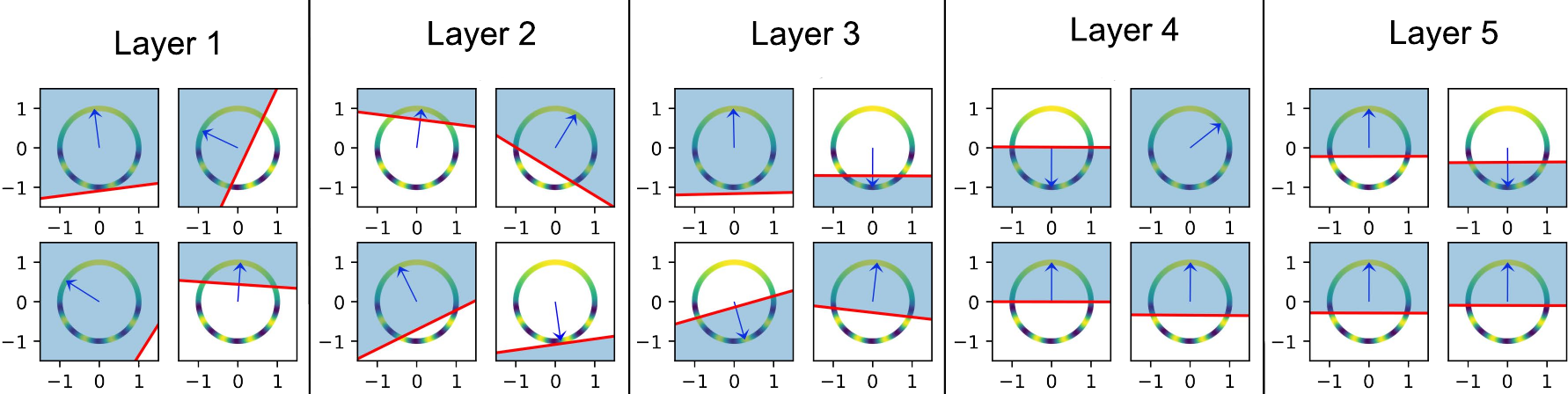}
    \caption{DLGN active regions for a subset of neurons in each layer at end of training}
    \label{fig:DLGN_active_regions_final}
\end{figure*}
    
\begin{figure*}[t]
    \centering
    \includegraphics[width=0.3\textwidth]{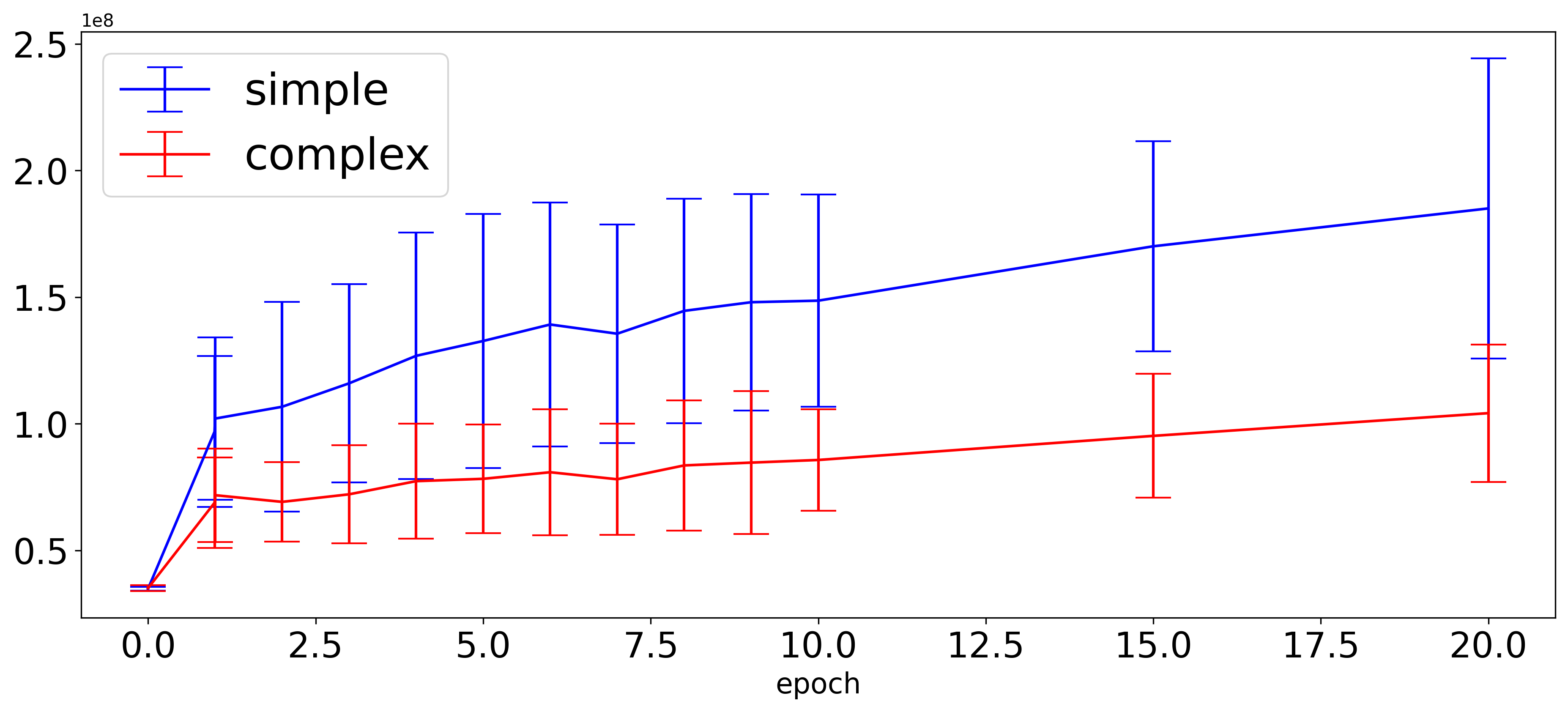}
    \includegraphics[width=0.3\textwidth]{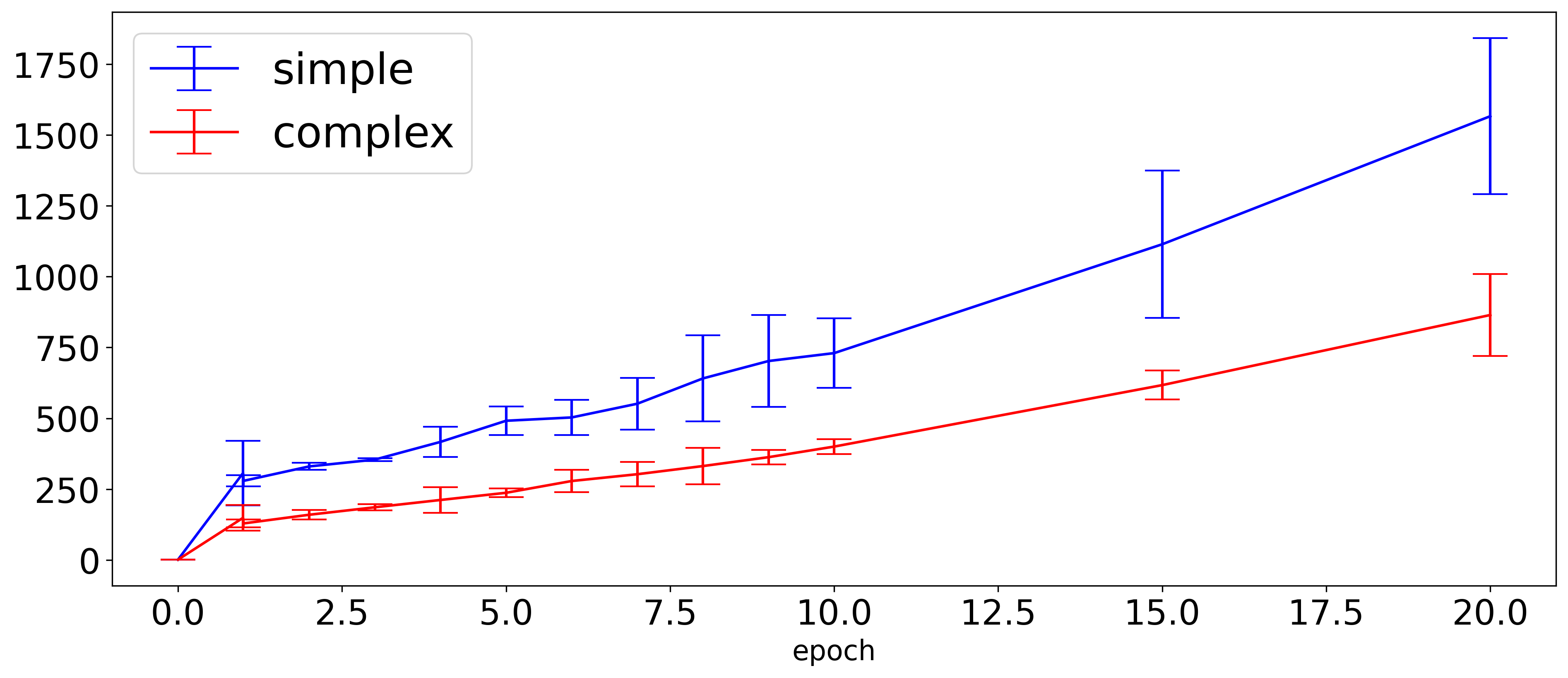}
    \includegraphics[width=0.3\textwidth]{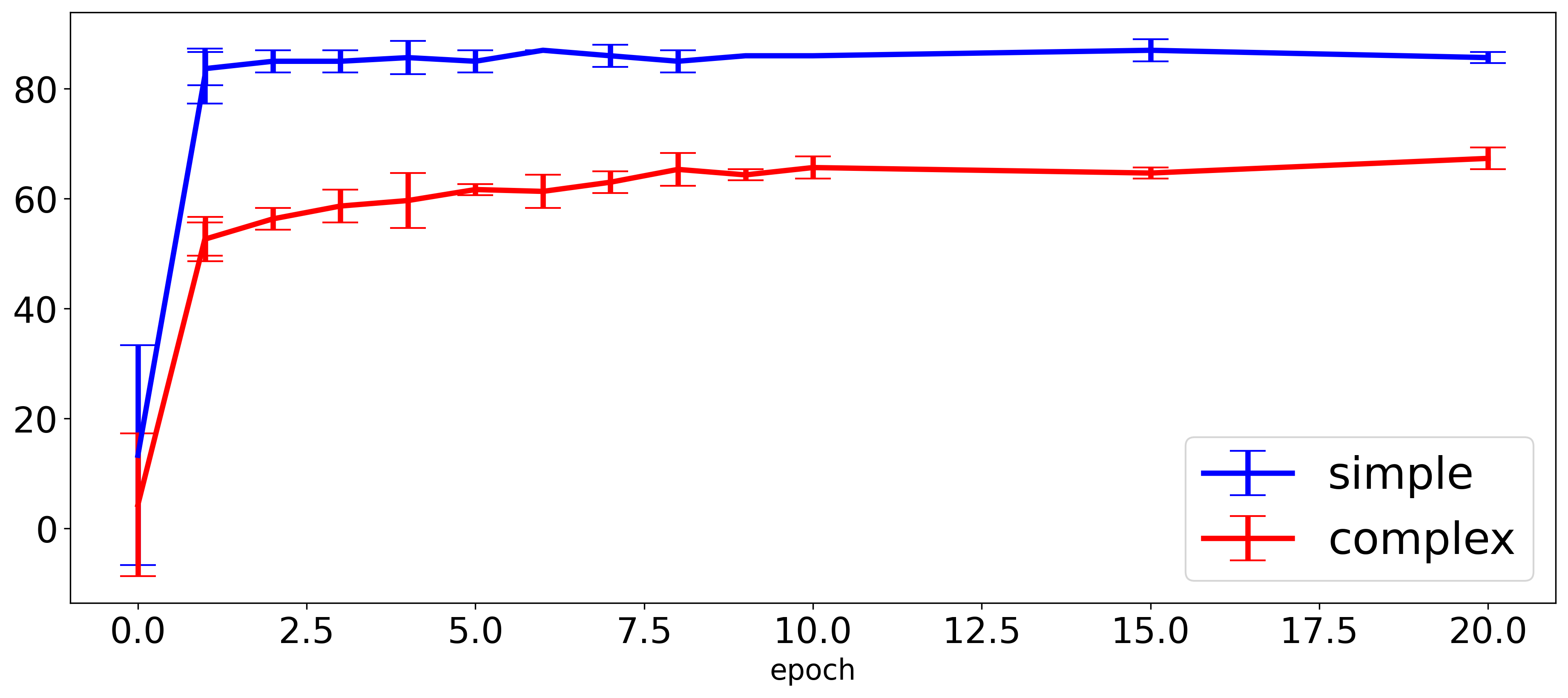}
    \caption{DLGN Results on Modified Fashion MNIST (Left) The average number of paths active on the simple and complex regions of input space. (Middle) Average value of the diagonal entry of the empirical NTK. (Right) Test accuracy. }
    \label{fig:DLGN_FMNIST}
\end{figure*}

The patterns revealed by the overlap kernel (See Figure \ref{fig:overlap_kernel_ReLU_DLGN}) show an interesting structure for both ReLU nets and DLGNs. In both cases, the number of paths active for the top half of the circle (corresponding to the lower frequency region for the target function) becomes much larger than the number of paths active for the other half. In both cases this broad structure is visible very early in the training (at the point where the training loss has only reduced from $0.5$ to $0.25$), and it stays more or less the same till the end of training (when the loss goes below $0.01$). This feature learning is indeed very beneficial for learning, as the version of DLGN where the gating model $f_\pi$ is fixed to the random initialisation has much worse convergence (it takes more than 2000 epochs as opposed to only about 200 epochs for ReLU nets and DLGN). 

Note that this is distinct from the phenomenon of spectral bias in neural nets (e.g. \cite{Basri+19}) which says: \textit{fixed infinite width NTK has low frequency sinusoids as eigen vectors with larger eigen values, and hence gradient descent prioritises learning low-frequency components of the target function}. The phenomenon we observe here happens when the (empirical) neural tangent kernel is far from being stationary -- See Figure \ref{fig:NTK_ReLU_DLGN}. A fundamental question in the field asks : \textit{when are neural networks better than fixed kernel machines?}. The standard answer for this is usually along the lines of -- \textit{the empirical NTK changes during training to become better adapted to the task at hand} \cite{Atanasov+22, Baratin+21, Fort+20}.  This does seem to happen, but the mechanism by which this can happen is largely unknown except for a few special cases where the target function is an arbitrary function of a one-dimensional projection of the input $\x$ (\cite{Damian22, Shi22, Daniely20}). Gradient descent (or any first order method) at any given iteration simply moves parameters in accordance with the current NTK features -- no force propels gradient descent to make the NTK itself better-suited for the task.  The overlap kernel gives us a plausible mechanism behind the change of NTK during training. We hypothesize that the active path regions change during training thus causing the NTK to change. 

Figure \ref{fig:preds} gives the learned model at different iterations and shows that the model learns the lower frequency half of the input space in just a few epochs, while the other half takes more than 100 epochs. An interesting result happens when the gating model of a DLGN is fixed at initialisation, under this setting no feature learning can happen as the active path regions are fixed. This causes slow convergence, and the learned model has high frequency artifacts on the upper half of the input space even after 8000 epochs. In our opinion, this proves that feature learning does happen in neural nets and is an essential requirement for efficient learning. Leading to the next question of - \textit{how are these features/path active regions learnt during training?}

Answering the above question is unfortunately extremely complicated for ReLU nets due to their complicated path active regions. DLGNs on the other hand have compactly representable active path regions. While analysing a new architecture like DLGNs seems to have less utility, we appeal to the closeness of performance of DLGNs with ReLU nets on small datasets like CIFAR10, and the similarity of overlap/NTK/prediction model dynamics (Figures \ref{fig:overlap_kernel_ReLU_DLGN}-\ref{fig:preds}) to argue that DLGNs form a faithful surrogate for ReLU nets and can be placed mid-way between ReLU nets and Deep linear networks. They can capture non-linear feature evolution unlike DLNs. while still being much more tractable than ReLU nets. 

As each neuron in DLGN is active exactly in a half-space, this allows for a simple way to visualise the neurons. In Figure \ref{fig:DLGN_active_regions_final} we visualise a representative subset (4 out of 16) of neurons in all 5 layers after training respectively. (The appendix contains the full set of activation regions including the neurons at initialisation). This immediately reveals the reason behind the difference in number of active paths in the upper and lower halves of the input space (the top right of Figure \ref{fig:overlap_kernel_ReLU_DLGN}(c)  is much brighter than the bottom left). A significant fraction of neurons (especially in the last layer) all choose to become active in the top half-space which contains the lower frequency part of the target function (3 out of 4 activation regions in layer 5 neurons of the trained DLGNs are approximately the top half of the input-space -- see Figure \ref{fig:DLGN_active_regions_final}). This observation underscores the interpretability of DLGNs as they demonstrate a preference for specific regions in the input space, particularly in capturing the lower-frequency features of the target function which is not possible to visualise for ReLU networks.

\section{Gradient Descent as a Resource Allocator}
\label{sec:GD-resource}
The viewpoint of neural networks as mixture of experts with path number of experts allows for a natural viewpoint where paths are resources. Each path specializes in one part of the input space -- for DLGNs this would be the intersection of half-spaces corresponding to its neurons.  The main goal in training then is to allocate paths to different regions of the input space based on the extra value it adds to the current aggregate model. For example, in the simple 2-dimensional regression problem discussed in the previous section, most paths specialised themselves to focus on the top half region of the input space. This happens automatically as a consequence of gradient descent. 

This leads us to make one of our main conjectures regarding feature learning in neural networks -- the activation regions of the experts/paths move preferentially to cover low-frequency/simpler regions of the target function in the input space. 

We provide further validation of the above hypothesis based on a variant of the Fashion MNIST(\cite{xiao17}) classification task. The images are used as-is. The labels of images from the first five classes (\texttt{T-shirt} to \texttt{Coat}) were assigned labels 1 through 5. The labels of the images from the next 5 classes (\texttt{Sandal} to \texttt{ankle boot}) are modified as follows. The images within each of these classes were sub-divided further into 5 groups based on a type of clustering and the class labels from 6 to 10 were randomly assigned to these 25 clusters. The net effect of this modification is that the `true' labelling function is much more complex in the input space corresponding to \texttt{Sandal} to \texttt{ankle boot}, than the space corresponding to \texttt{T-shirt} to \texttt{Coat} images.

Figure \ref{fig:DLGN_FMNIST} supports our conjecture -- the number of paths allocated to the simple part of the input space is consistently higher than the complex part. We argue that this is an innate weakness of gradient descent that allocates more resources for an easier job. This is another facet of the simplicity bias of gradient descent that is not captured by spectral arguments on the NTK, as the NTK is far from constant during the training process (Further details of these experiments are given  in the appendix).

Investigating this defect of gradient descent, and workarounds using some other architecture aware optimisation algorithm is an interesting direction of future work.

\section{Discussion and Conclusion}

In this paper, we framed the neural network model as a mixture of a large number of simple experts and put both ReLU nets and Deep Linearly Gated Networks in the same category of models. We introduced novel tools to study feature learning in neural nets, and made the claim that features in neural nets correspond to paths through the network and they represent intersections of half-spaces in the input space for DLGNs. We studied feature learning on a simple synthetic task and identified an important phenomenon of most neurons taking responsibility for the `easy' region of the input space. This intriguing observation was vividly visualized in DLGNs, providing a compelling argument for their enhanced interpretability compared to ReLU networks. Studying these phenomena further is an exciting direction for further research.


\bibliography{neural_nets_theory_phenomenology}



\newcommand{\summ}[2]{\sum_{#1}^{#2}}
\onecolumn

\section*{\centering {Appendix}}
\setcounter{section}{0}
\section{DLGN Illustration}
\begin{figure*}[ht]
    \centering
    \includegraphics[scale = .9,width=0.45\textwidth]{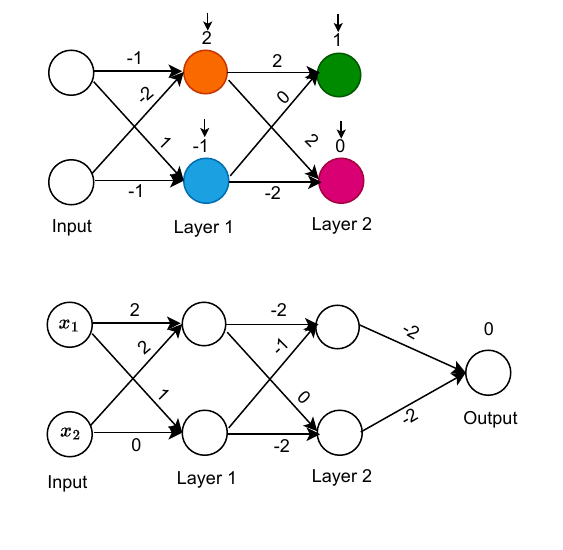}
    \includegraphics[scale = .9,width=0.45\textwidth]{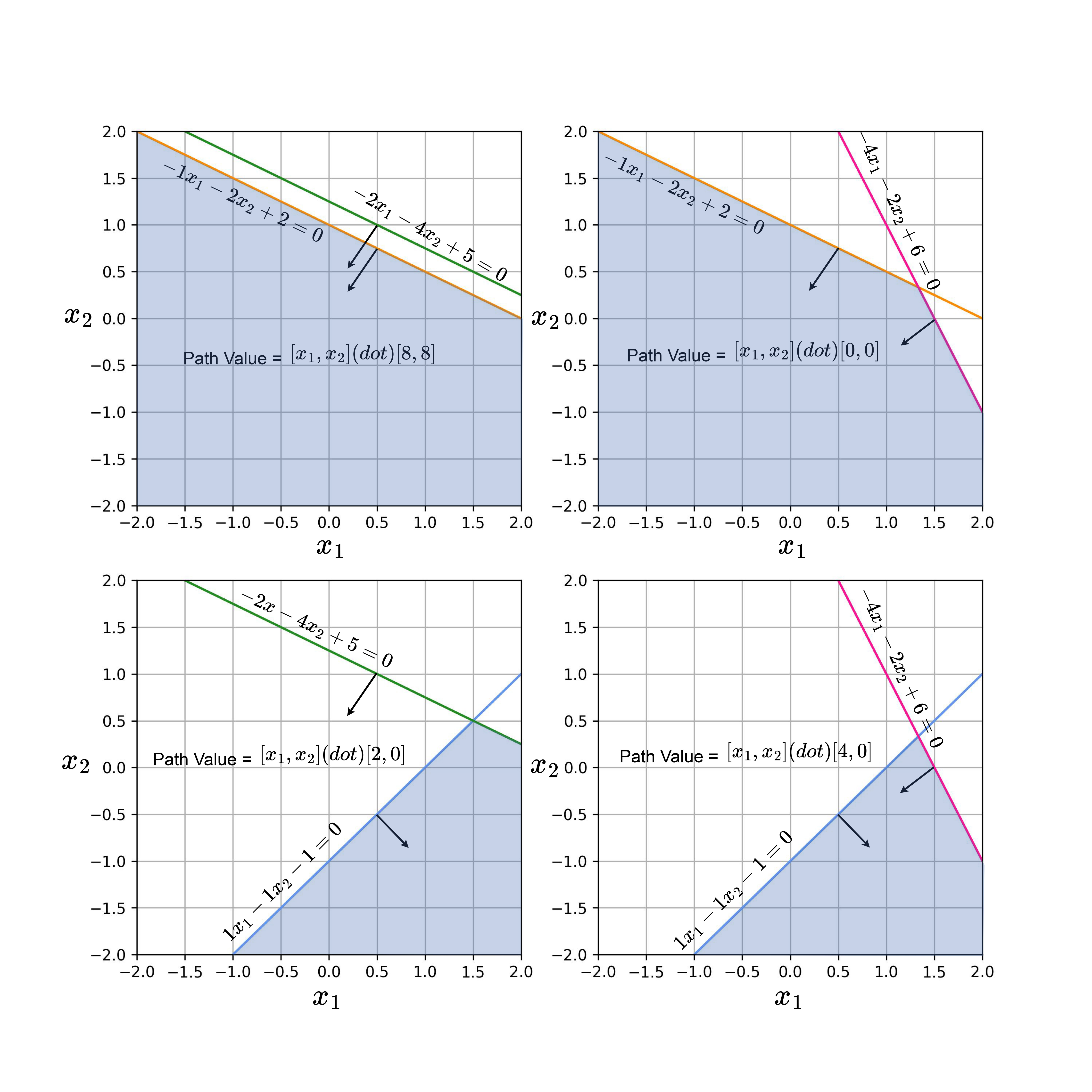}
    \caption{Left: DLGN network with the gating model above and simple expert model below. Right: The four paths are each active in a different region of the input space given by the intersection of the appropriately coloured half-spaces. The value of each path is either a linear function of the input (for vanilla-DLGN), or a constant (for DLGN-PWC, in which case the $[x_1,x_2]$ in the value expression are replaced by $[1,1]$). }
    \label{fig:DLGN_paths}
\end{figure*}
\section{Comaprison results between DLGN and ReLU}
Following is the result for comparison between the performance between the DLGN and ReLU for standard architecture and datasets.
\begin{table*}[ht]
  \centering
  \begin{tabular}{|c|c|c|c|c|}
    \hline
    \multicolumn{1}{|c|}{Architecture} & \multicolumn{2}{c|}{CIFAR10} & \multicolumn{2}{c|}{CIFAR100}\\
    \cline{2-5}
    & Train Accuracy & Test Accuracy & Train Accuracy & Test Accuracy \\
    \hline
    Resnet 34(ReLU) & \multicolumn{1}{c|}{100.0\%} & \multicolumn{1}{c|}{91\%} & \multicolumn{1}{c|}{100\%} & \multicolumn{1}{c|}{56.68\%} \\
    \hline
    Resnet 34(DLGN) & \multicolumn{1}{c|}{100.0\%} & \multicolumn{1}{c|}{86\%} & \multicolumn{1}{c|}{100\%} & \multicolumn{1}{c|}{51.6\%}  \\
    \hline
    Resnet 110(ReLU) & \multicolumn{1}{c|}{100.0\%} & \multicolumn{1}{c|}{94\%} & \multicolumn{1}{c|}{100\%} & \multicolumn{1}{c|}{73\%} \\
    \hline
    Resnet 110(DLGN) & \multicolumn{1}{c|}{100.0\%} & \multicolumn{1}{c|}{89\%} & \multicolumn{1}{c|}{100\%} & \multicolumn{1}{c|}{67\%}  \\
    \hline
  \end{tabular}
  \caption{Experimental Results for DLGN and ReLU nets with ResNet architecture on CIFAR10 and CIFAR100}
\label{tab:resnet_comparison}
\end{table*}

\section{Experimental Setup}
Results shown in Table \ref{tab:MoE_formulations} was done on CIFAR-10 dataset. Same Convolution architectures(number of convolution layers and number of filters) were used for all the experiments. Following are the hyperparameters and details used for the same.
\subsection{Deatails of experiments performed on CIFAR10}
\subsubsection{Hyperparameters}
\begin{itemize}
    \item Number of Convolution Layers = 5
    \item Number of filters in each layer = 26
    \item Optimizer : Adam
    \item Learning rate = 2e-4
\end{itemize}

\subsubsection{Architectures}
\begin{itemize}
    \item ReLU network : 5 Convolution layers with ReLU activation in each layer, followed by Global Average Pooling, followed by 1 Dense layer with 64 neurons.
    \item DLN : 5 Convolution layers, followed by Global Average Pooling, followed by 1 Dense layer with 64 neurons.
    \item DLGN, DLGN-PWC : 5 Convolution layers, followed by Global Average Pooling, followed by 1 Dense layer with 64 neurons.
\end{itemize}

\subsection{Details of Experiments performed on Fashion MNIST}
Data set was modified as described in the main paper. ReLU and DLGN-PWC models were used for ReLU and DLGN results respectively. The fashion MNIST images were flattend into $1D$ vectors and given as input to respective models.
\subsubsection{Hyperparameters}
\begin{itemize}
    \item Number of hidden layers = 5
    \item Number of nodes in each layer = 128
    \item Optimizer : Adam
    \item Learning rate = 2e-4
\end{itemize}
The overlap, NTK and accuracy graphs on Modified fashion MNIST are shown in corresponding sections.

\clearpage
\section{DLGN active regions for circle dataset}
\begin{figure*}[ht]
    \centering
    \includegraphics[width=\textwidth]{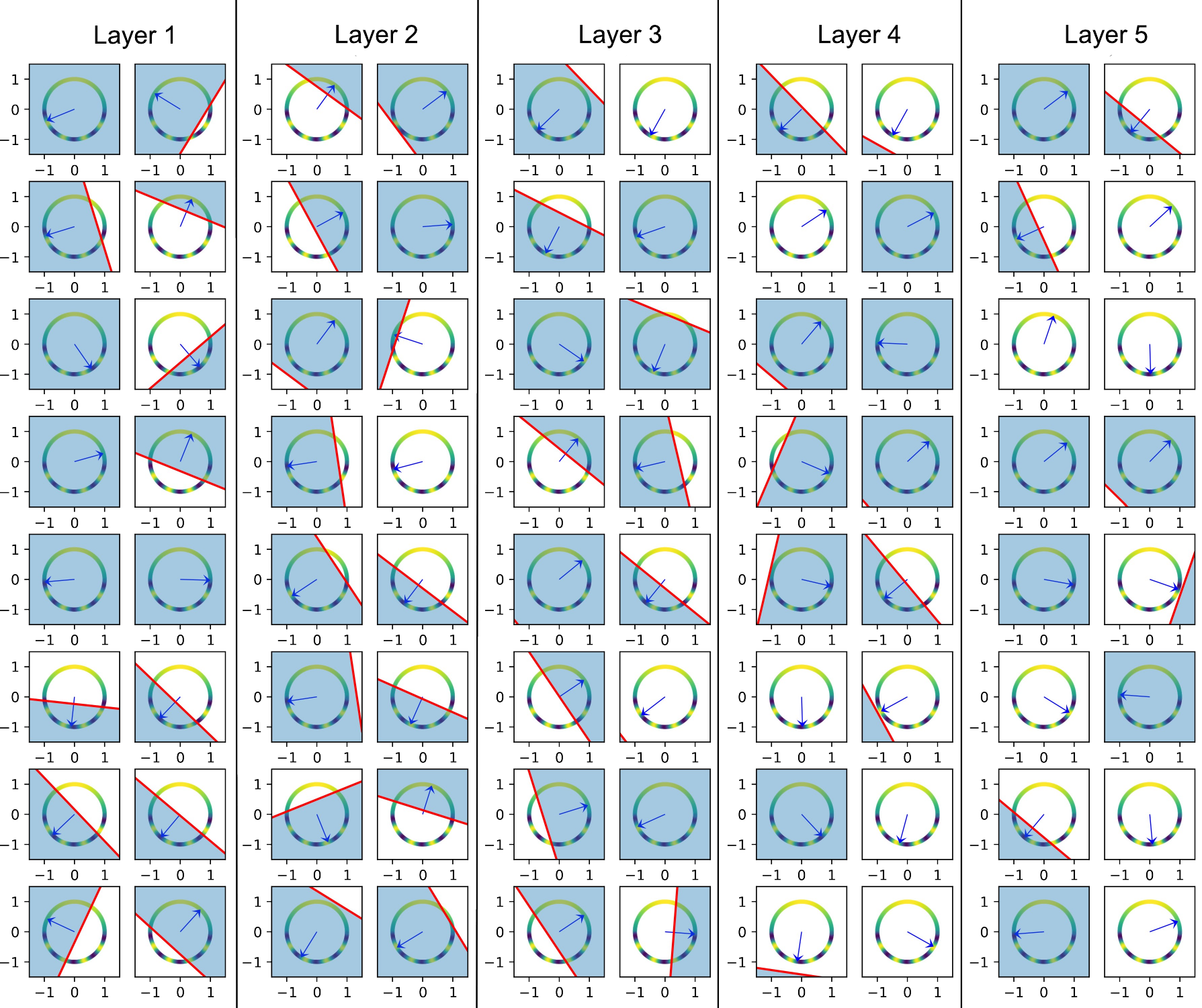}
    \caption{DLGN active regions in each layer at initialisation}
    \label{fig:DLGN_active_regions_init_full}
\end{figure*}
\begin{figure*}[ht]
    \centering
    \includegraphics[width=\textwidth]{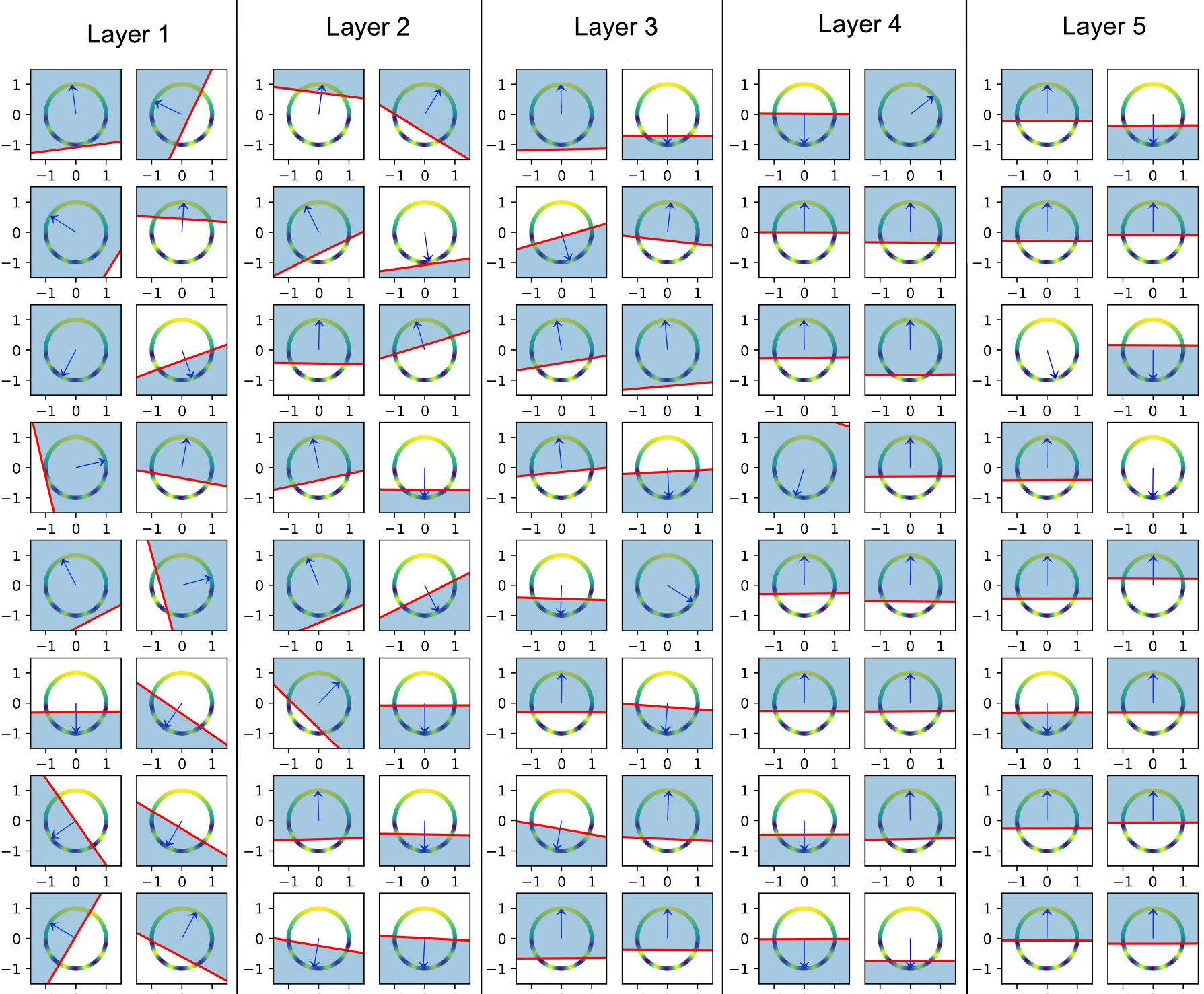}
    \caption{DLGN active regions in each layer at end of training}
    \label{fig:DLGN_active_regions_final_full}
\end{figure*}
\clearpage
\newpage

\section{Figures related to  Fashion MNIST}
\begin{figure*}[!ht]
    \centering
    \begin{subfigure}{.7\textwidth}
        \includegraphics[width=\linewidth]{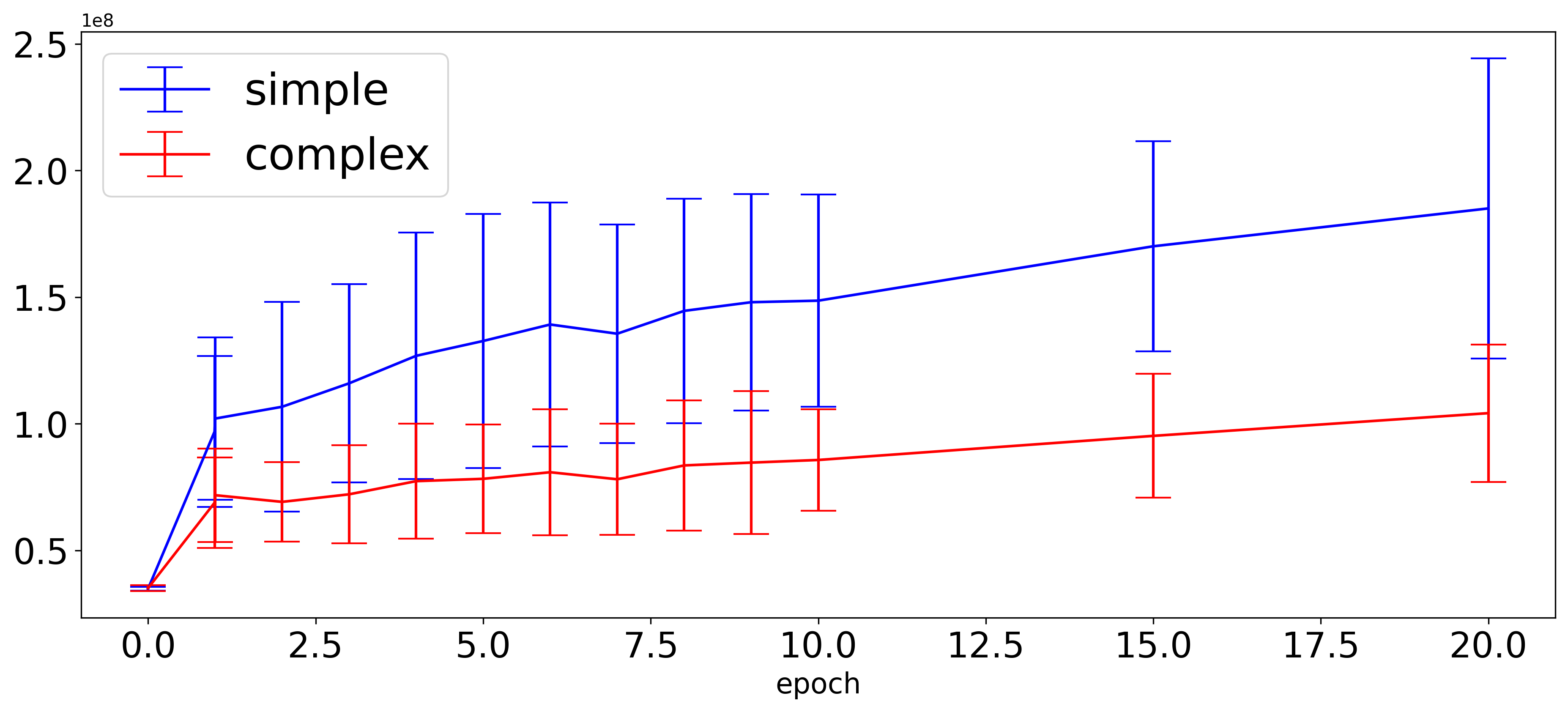}
        \caption{Average number of paths active on the simple and complex regions of the input space.} 
    \end{subfigure}
    \par
    \begin{subfigure}{.7\textwidth}
        \includegraphics[width=\linewidth]{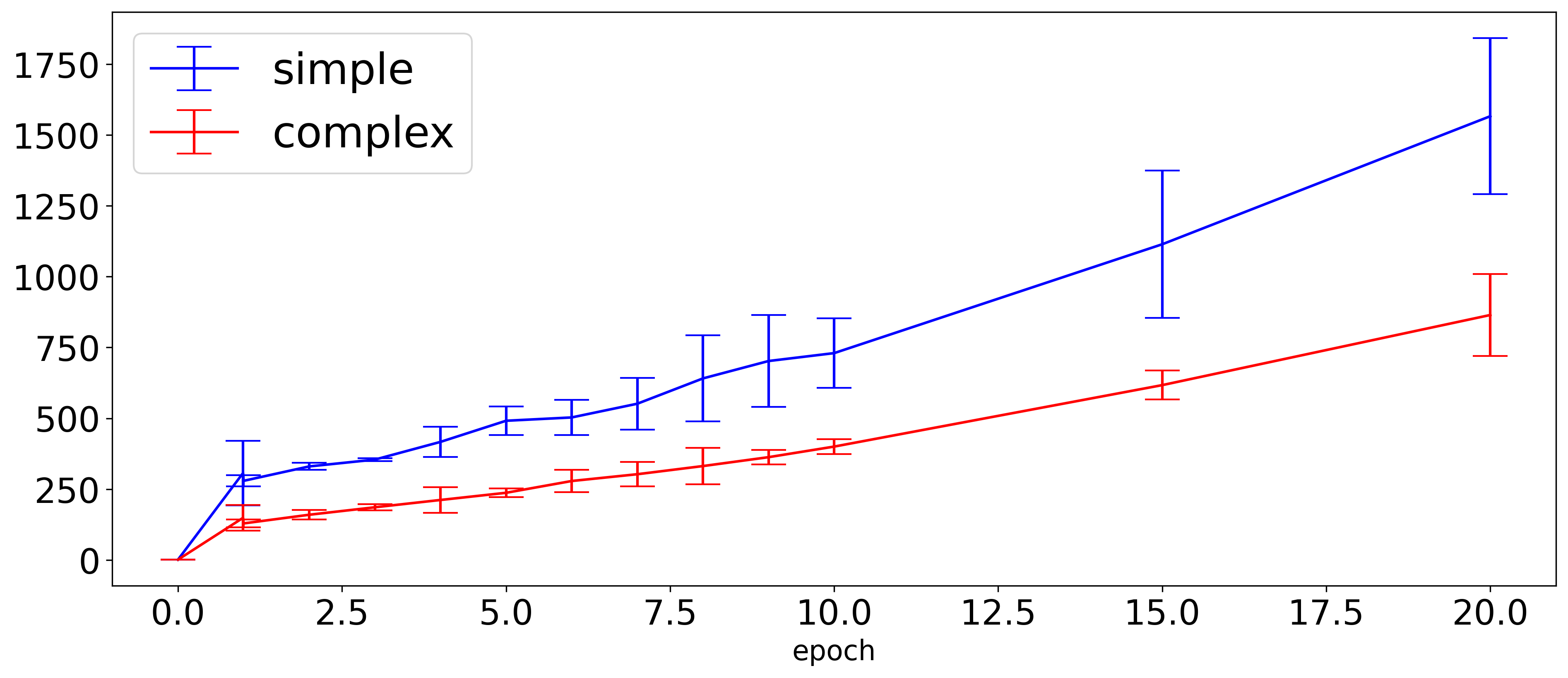}
        \caption{Average value of the diagonal entry of the empirical NTK} 
    \end{subfigure}
    \par
    \begin{subfigure}{.7\textwidth}
        \includegraphics[width=\linewidth]{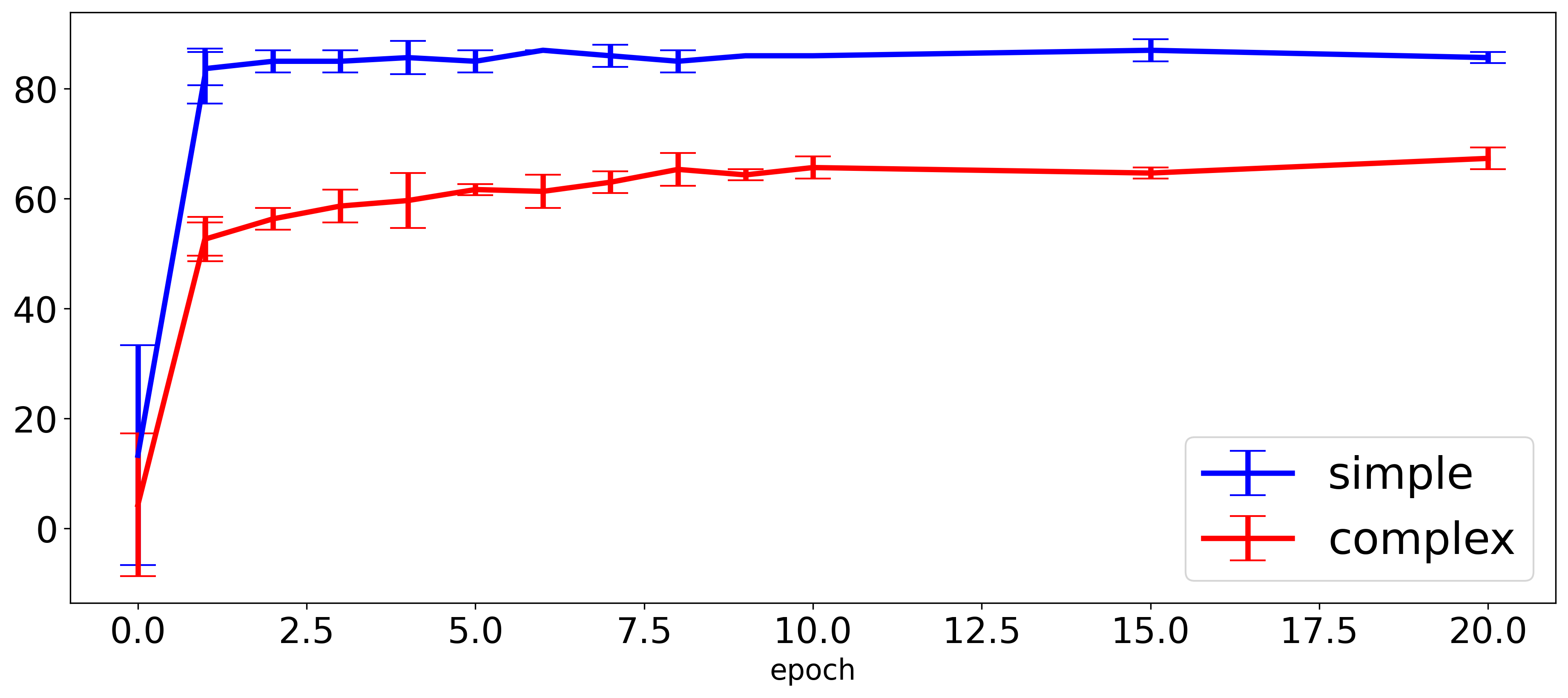}
        \caption{Average Test accuracy} 
    \end{subfigure}
    \caption{DLGN Results on Modified Fashion MNIST. Average is done over 3 runs of the experiment.}
    \label{fig:DLGN_FMNIST_appendix}
\end{figure*}

\begin{figure*}[!ht]
    \centering
    \begin{subfigure}{.7\textwidth}
        \includegraphics[width=\linewidth]{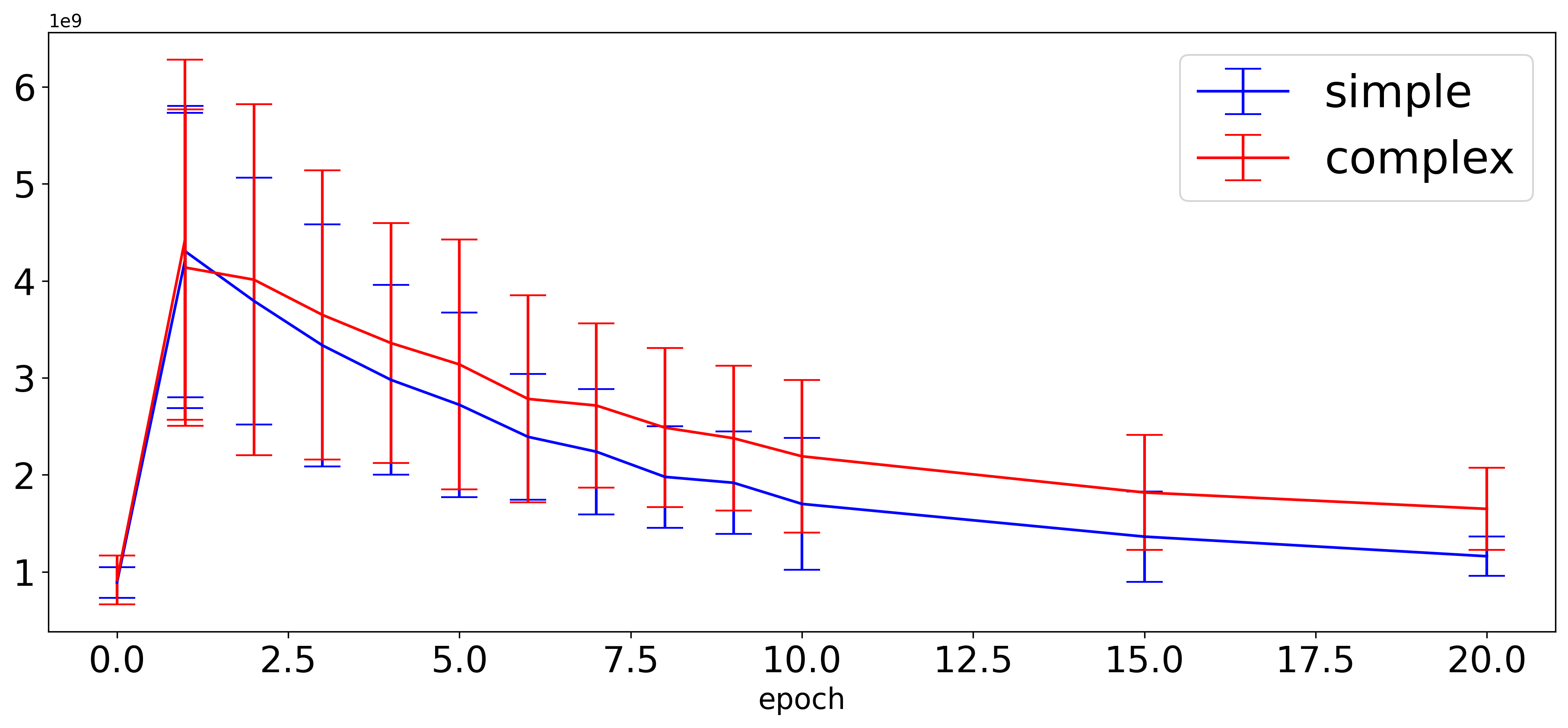}
        \caption{Average number of paths active on the simple and complex regions of the input space.} 
    \end{subfigure}
    \par
    \begin{subfigure}{.7\textwidth}
        \includegraphics[width=\linewidth]{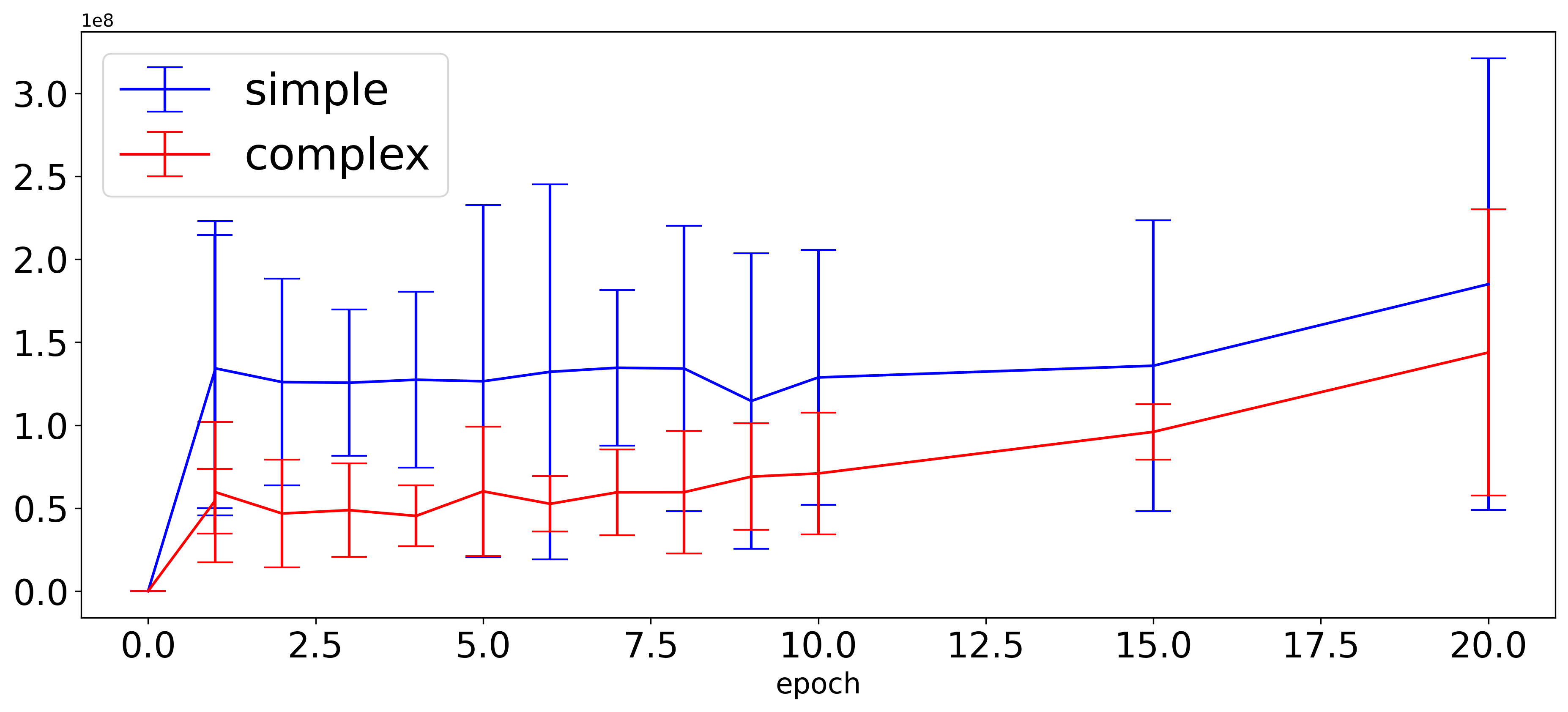}
        \caption{Average value of the diagonal entry of the empirical NTK} 
    \end{subfigure}
    \par
    \begin{subfigure}{.7\textwidth}
        \includegraphics[width=\linewidth]{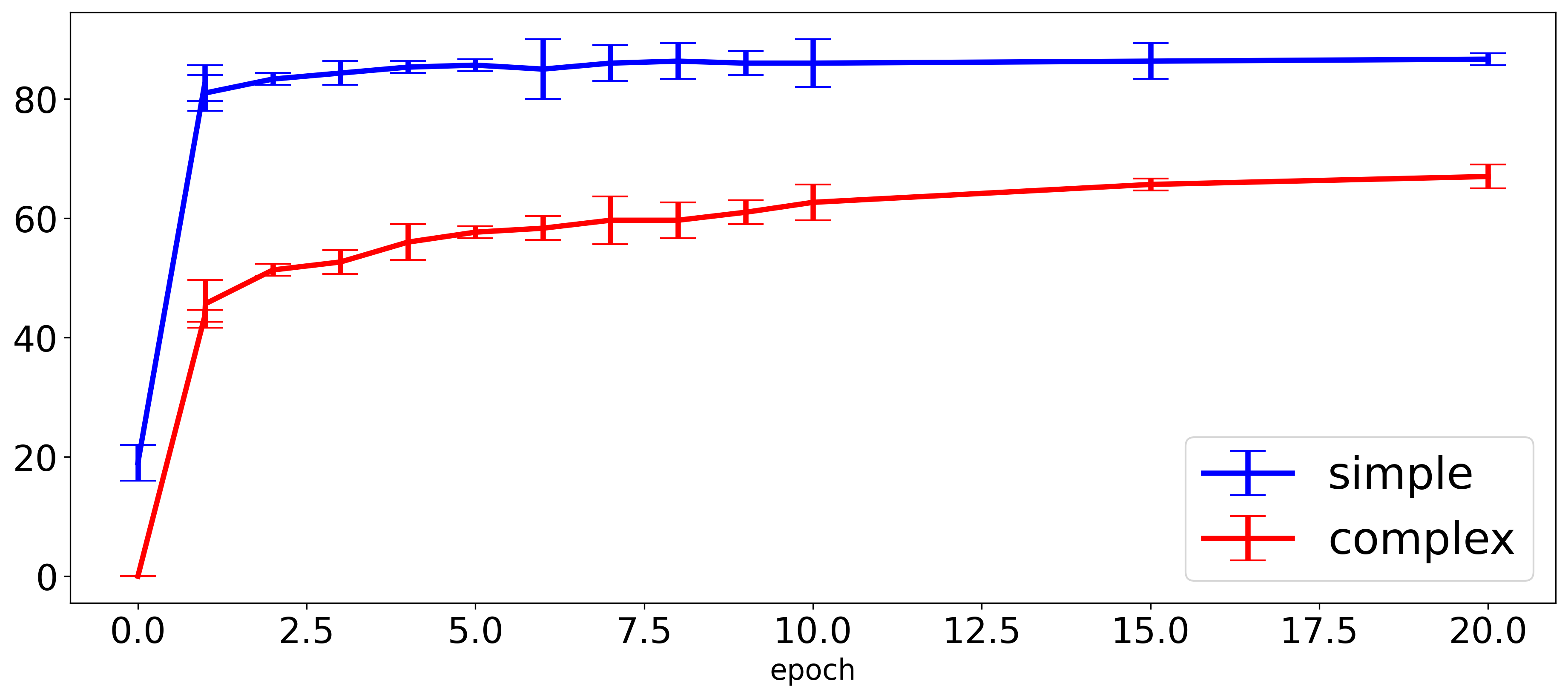}
        \caption{Average Test accuracy} 
    \end{subfigure}
    \caption{ReLU Results on Modified Fashion MNIST. Average is done over 3 runs of the experiment.}
    \label{fig:ReLU_FMNIST_appendix}
\end{figure*}
\newpage
\subsection{Overlap and NTK for Fashion MNIST}
Overlap kernels and NTK for Fashion MNIST is shown in the following figures. We have shown the kernels for 3 critical points. First at initialization, second at middle point where the loss on easy part is $\approx72\%$. This is to capture a point where the "easy" mode is almost learned(72 out of 85) but the hard mode is still mostly not learned yet.
 \begin{figure*}[ht]
    \centering
    \begin{subfigure}{.4\textwidth}
        \includegraphics[width=\linewidth]{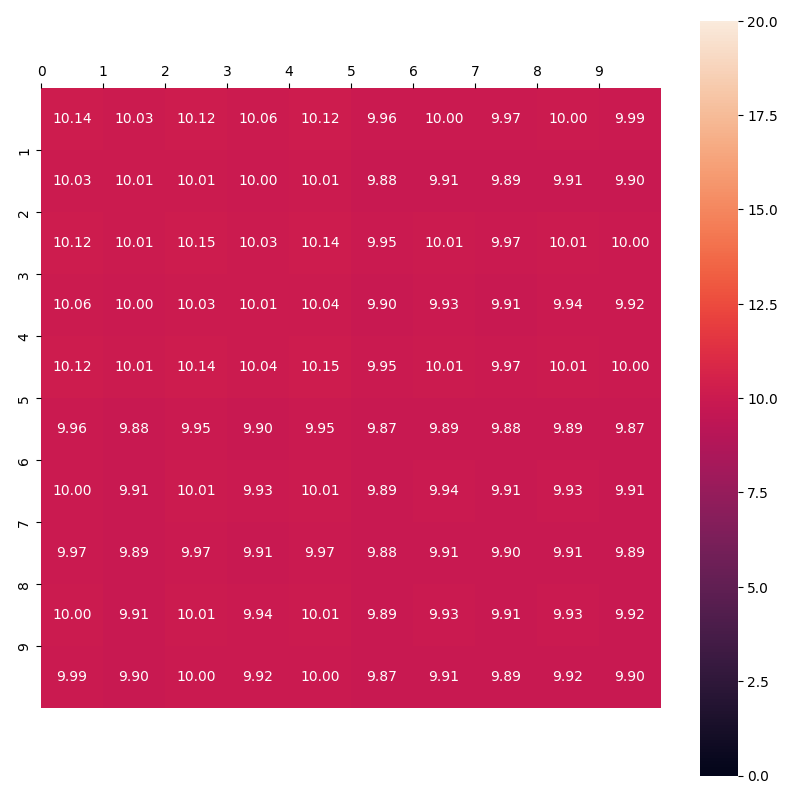}
        \caption{The (trace-normalised) overlap kernel for the DLGN visualised as matrix at initialization}
    \end{subfigure}
    \begin{subfigure}{.4\textwidth}
        \includegraphics[width=\linewidth]{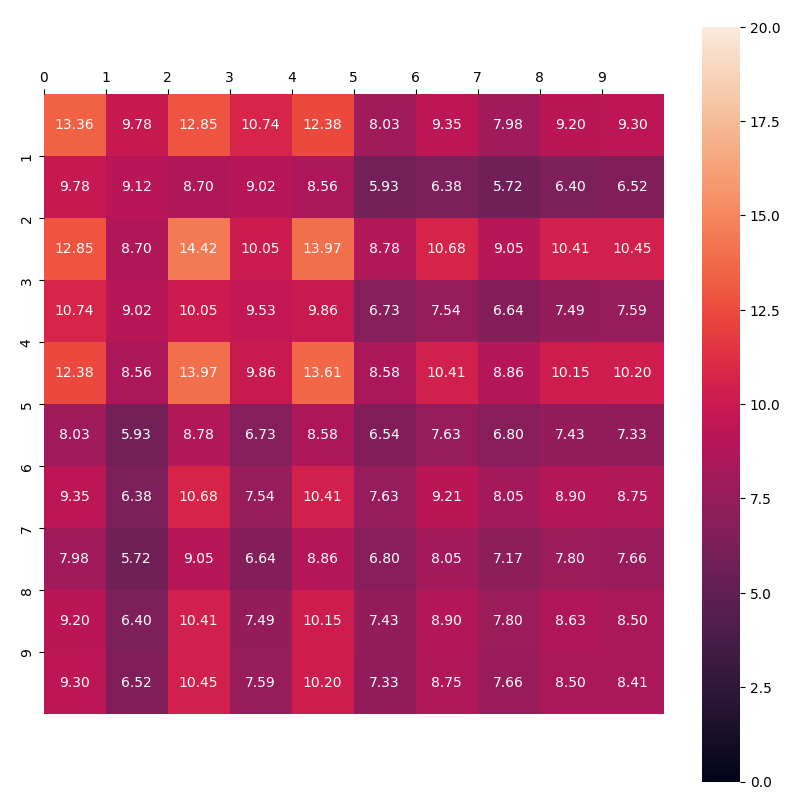}
        \caption{The (trace-normalised) overlap kernel for the DLGN visualised as matrix at middle.}
    \end{subfigure}
    
    \begin{subfigure}{.4\textwidth}
        \includegraphics[width=\linewidth]{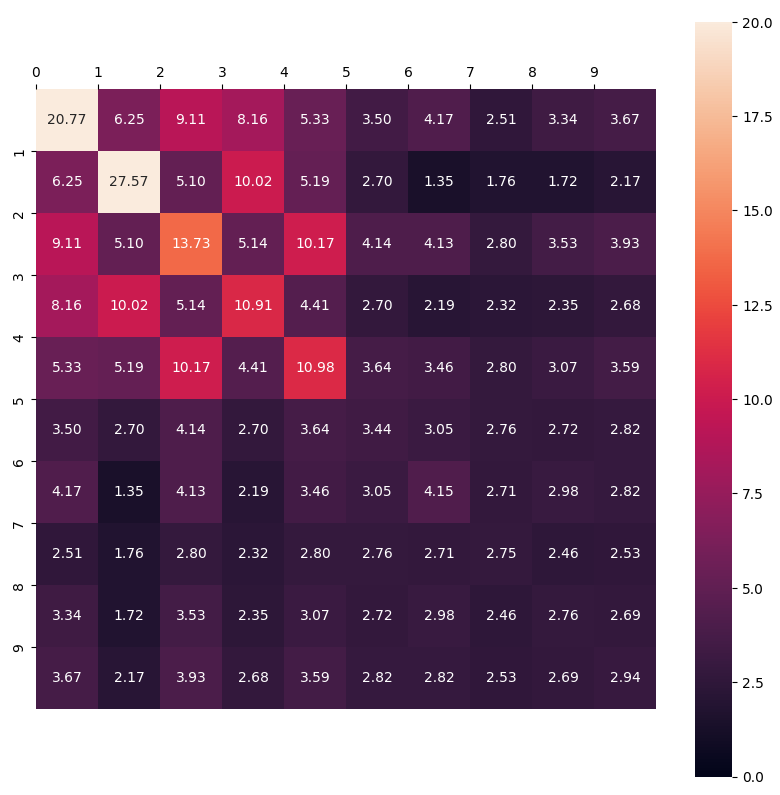}
        \caption{The (trace-normalised) overlap kernel for the DLGN visualised as matrix at final epoch.}
    \end{subfigure}
    \caption{The (trace-normalised) overlap kernels for the DLGN visualised as matrix at various epochs. The data points are ordered such that 0-10 scale in the heatmap represents class 1 to class 10(for ex. scale 0-1 represents all data points from class 1) respectively. 50 data points from each class(totalling 500 samples) are taken and then averaged over each class to build the overlap kernel of size 10x10 which is shown in the above heatmaps.}
    
\end{figure*}
\begin{figure*}[ht]
    \centering
    \begin{subfigure}{.4\textwidth}
        \includegraphics[width=\linewidth]{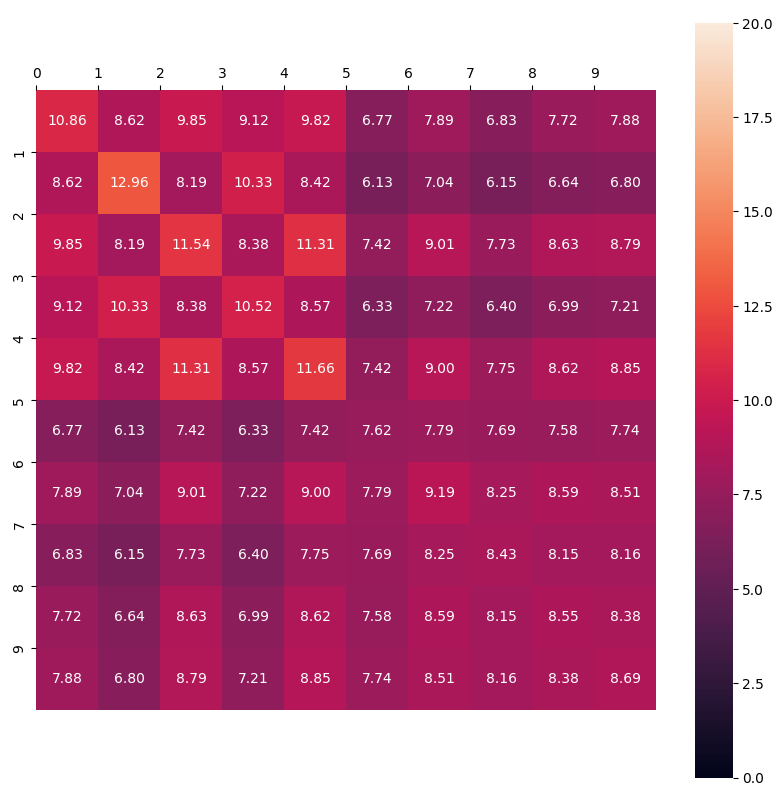}
        \caption{The (trace-normalised) overlap kernel for the MLP visualised as matrix at initialization}
    \end{subfigure}
    \begin{subfigure}{.4\textwidth}
        \includegraphics[width=\linewidth]{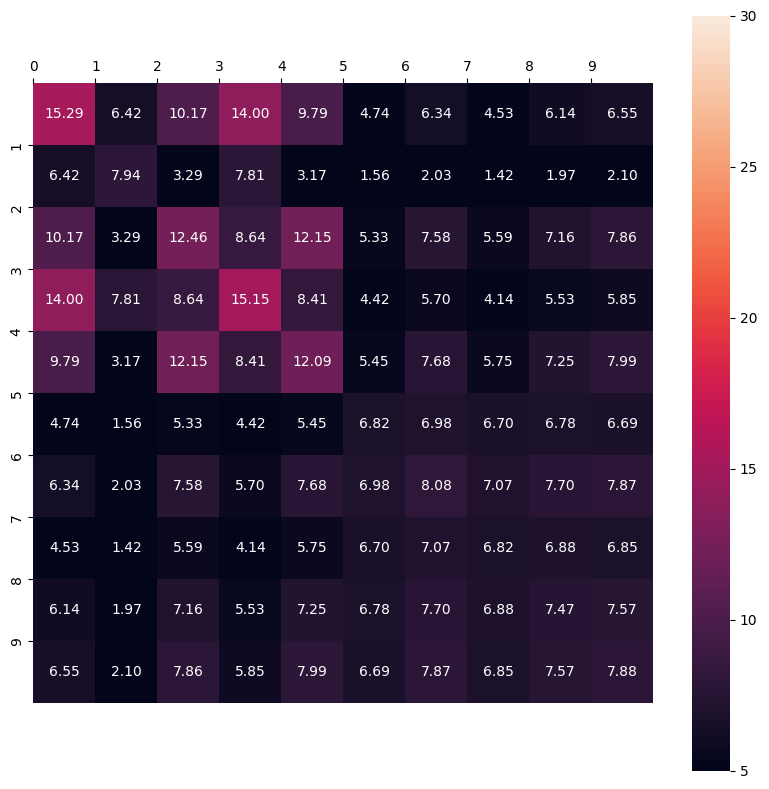}
        \caption{The (trace-normalised) overlap kernel for ReLU visualised as matrix at middle.}
    \end{subfigure}
    
    \begin{subfigure}{.4\textwidth}
        \includegraphics[width=\linewidth]{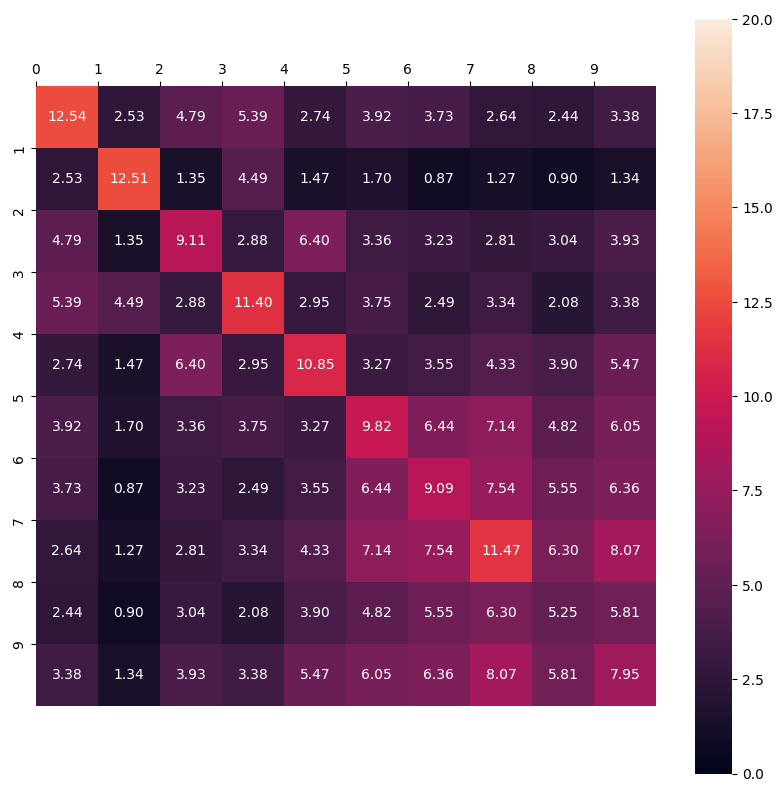}
        \caption{The (trace-normalised) overlap kernel for the ReLU visualised as matrix at final epoch.}
    \end{subfigure}
    \caption{The (trace-normalised) overlap kernels for the ReLU visualised as matrix at various epochs. The data points are ordered such that 0-10 scale in the heatmap represents class 1 to class 10(for ex. scale 0-1 represents all data points from class 1) respectively. 50 data points from each class(totalling 500 samples) are taken and then averaged over each class to build the overlap kernel of size 10x10 which is shown in the above heatmaps.}
    
\end{figure*}

\begin{figure*}[ht]
    \centering
    \begin{subfigure}{.4\textwidth}
        \includegraphics[width=\linewidth]{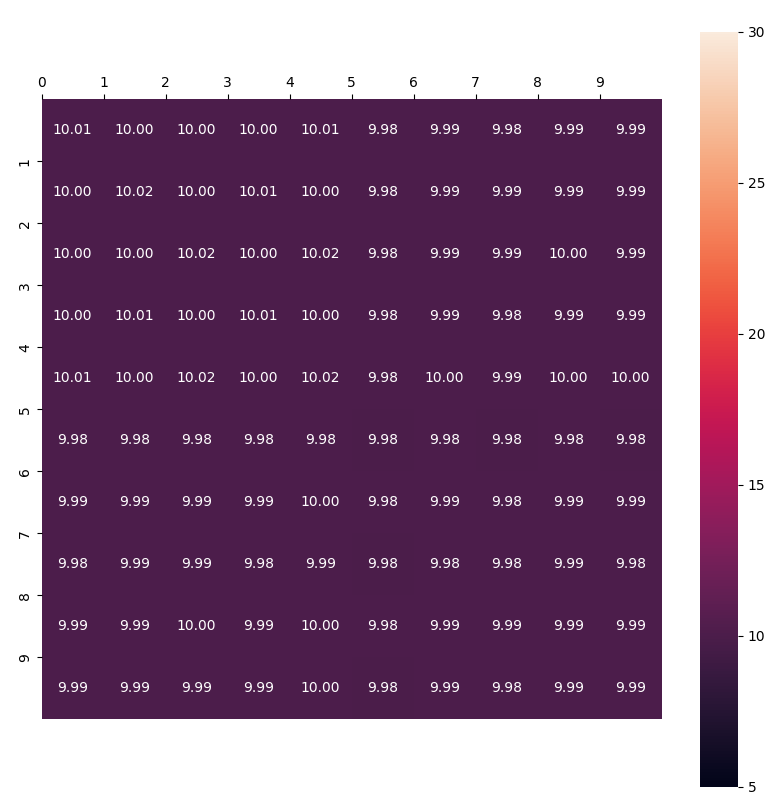}
        \caption{The (trace-normalised) NTK for the DLGN visualised as matrix at initialization}
    \end{subfigure}
    \begin{subfigure}{.4\textwidth}
        \includegraphics[width=\linewidth]{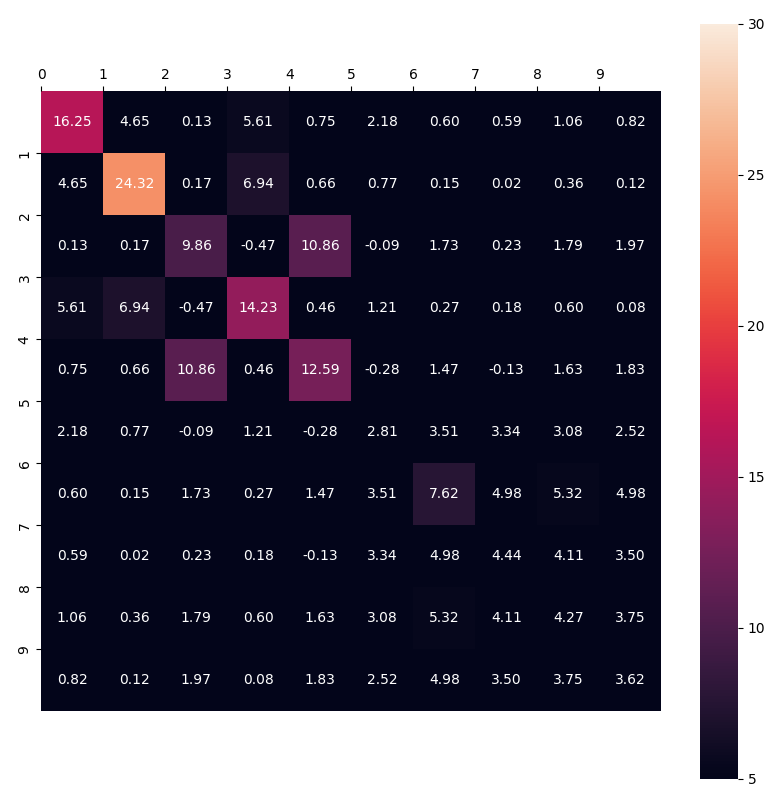}
        \caption{The (trace-normalised) NTK for the DLGN visualised as matrix at middle.}
    \end{subfigure}
    
    \begin{subfigure}{.4\textwidth}
        \includegraphics[width=\linewidth]{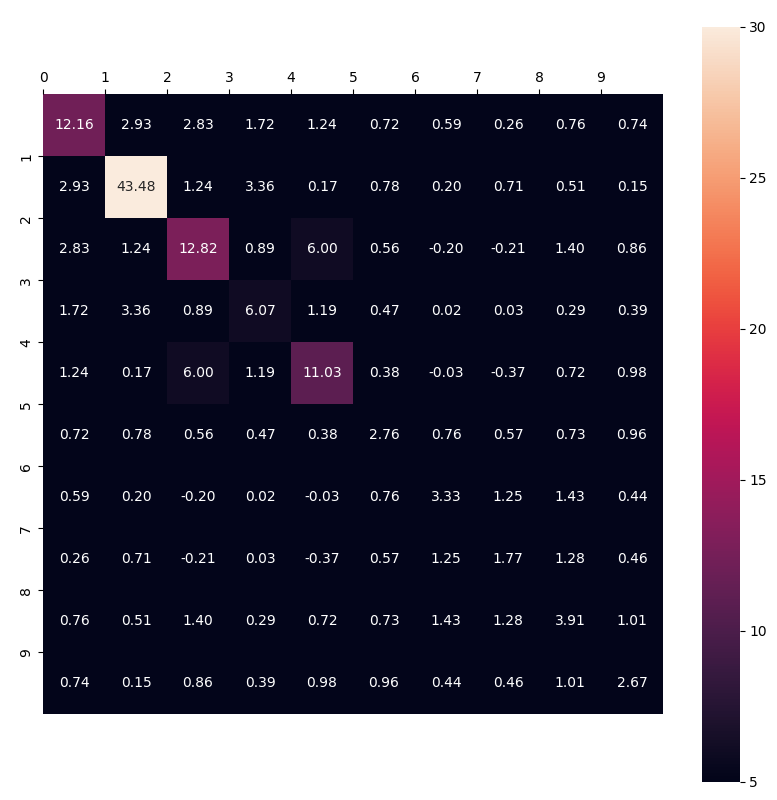}
        \caption{The (trace-normalised) NTK for the DLGN visualised as matrix at final epoch.}
    \end{subfigure}
    \caption{The (trace-normalised) NTK for the DLGN visualised as matrix at various epochs. The data points are ordered similiar to as described for the overlap kernels.}
    
\end{figure*}
\begin{figure*}[ht]
    \centering
    \begin{subfigure}{.4\textwidth}
        \includegraphics[width=\linewidth]{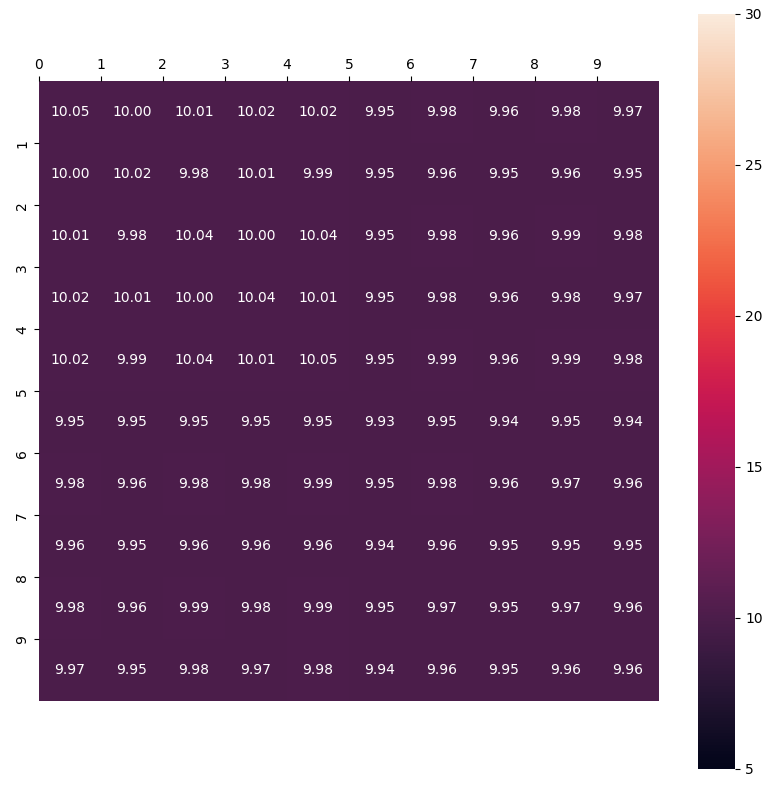}
        \caption{The (trace-normalised) NTK kernel for ReLU visualised as matrix at initialization}
    \end{subfigure}
    \begin{subfigure}{.4\textwidth}
        \includegraphics[width=\linewidth]{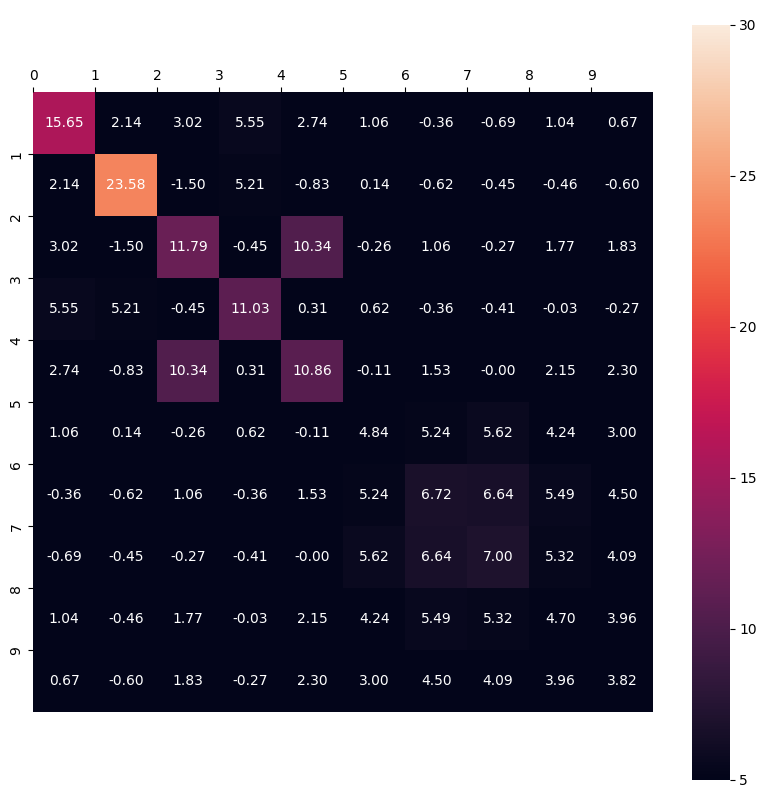}
        \caption{The (trace-normalised) NTK kernel for the ReLU visualised as matrix at middle.}
    \end{subfigure}
    
    \begin{subfigure}{.4\textwidth}
        \includegraphics[width=\linewidth]{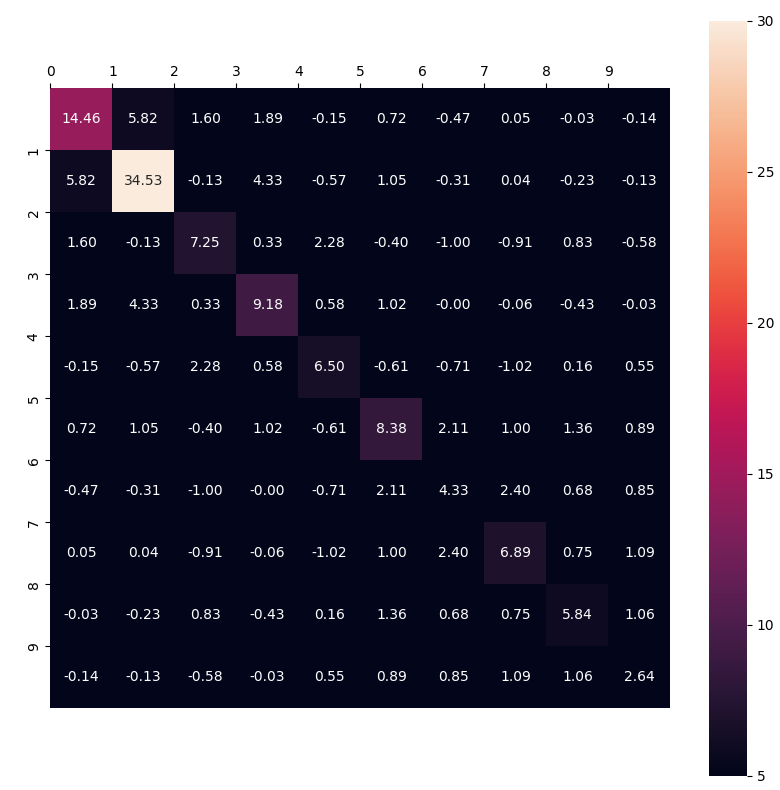}
        \caption{The (trace-normalised) NTK kernel for the ReLU visualised as matrix at final epoch.}
    \end{subfigure}
    \caption{The (trace-normalised) NTK for the DLGN visualised as matrix at various epochs. The data points are ordered similiar to as described for the overlap kernels.}
    
\end{figure*}

\clearpage
\newpage
\section{Proofs}
\subsection{Proof of  Theorem \ref{thm:dlgn-theorem}}
\begin{theorem*}
Let $W_2, \ldots, W_{L-1}$ be matrices of size $m \times m$. Let $W_1 \in \R^{m\times d}$ and $W_L \in \R^{1\times m}$. Let $U_1, \ldots, U_L$ be matrices of the same size as the corresponding $W$ matrices. For any path $\pi = (i_1, \ldots, i_{L-1}) \in [m]^{L-1}$, let the gating and expert models $f_\pi, g_\pi:\R^d \> \R$ be
\begin{align*}
f_\pi(\x) &= \prod_{\ell=1}^{L-1} \1\left(\boldeta_{\ell, i_\ell}(\x) \geq 0 \right) \\
g_\pi(\x) &= \left[ \prod_{\ell=2}^{L-1} u_{\ell,i_\ell,i_{\ell-1}}\right] u_{L,1,i_{L-1}} \bu_{1,i_1}^\top \x. \\
\end{align*}
Also, let $\boldeta, \h$ be collections of mappings such that
\begin{align}
    \boldeta_0(\x)=\h_0(\x) &= \x \nonumber \\ 
    \text{ for } \ell \in \{1,\ldots,L-1\} \qquad \boldeta_\ell(\x) &= W_\ell \boldeta_{\ell-1}(\x)   \nonumber \\
    \text{ for } \ell \in \{1,\ldots,L-1\} \qquad  \h_\ell(\x) &= \1(\boldeta_\ell(\x) \geq 0) \circ \left( U_\ell \h_{\ell-1}(\x) \right) \label{eqn:h-vector-defn}.
\end{align}
Then we have that
\begin{equation}
        \sum_{\pi \in [m]^{L-1}} f_\pi(\x) g_\pi(\x)  = U_L \left(\h_{L-1}(\x) \right)
        \label{eqn:to_prove}
\end{equation}
\end{theorem*}


We prove an intermediate lemma based on induction first.

\begin{lem}
Let $\ell \in \{1,2,\ldots, L-1\}$, and let $i_\ell \in [m]$
\begin{equation}
\h_{\ell,i_\ell}(\x) = \summ{i_0 = 1}{d} x_{i_0} \summ{i_1=1}{m}\summ{i_1=2}{m}\dotsm \summ{i_{\ell-1}=1}{m} \left(\prod_{k=1}^{\ell}  \1(\boldeta_{k,i_k}(\x)\geq0)  \right) \left(\prod_{k=1}^\ell u_{k, i_{k}, i_{k-1}}  \right)
\label{eqn:lem-induction}
\end{equation}
\end{lem}
\begin{proof}

We will prove by Induction.\\
Basis Step: When $\ell = 1$, the RHS of Equation \ref{eqn:lem-induction} becomes
\[
\h_{1,i}(\x) = U_{1,i}\h_{0}(\x) \1(\boldeta_{1, i}(\x) \geq 0) = U_{1,i}\x\1(\boldeta_{1, i}(\x) \geq 0)
\]
which is true by the definition of $\h$ from Equation \ref{eqn:h-vector-defn} and that $\h_0(\x)=\x$.

Assume Equation \ref{eqn:lem-induction} is true for some $\ell$. We now prove that it holds for $\ell+1$

From the definition of $\h$ in Equation \ref{eqn:h-vector-defn} and the induction assumption, for any $i_{\ell+1} \in [m]$
\begin{align*}
\h_{\ell+1,i_{\ell+1}}(\x) 
&=  \1(\eta_{\ell+1,i_{\ell+1}}(\x) \geq 0)  \left( \bu_{\ell+1,i_{\ell+1}} \h_{\ell}(\x) \right)  \\
&=  \1(\eta_{\ell+1,i_{\ell+1}}(\x) \geq 0)  \sum_{i_\ell =1}^m u_{\ell+1,i_{\ell+1},i_\ell} h_{\ell, i_\ell}(\x)   \\
&= \summ{i_0 = 1}{d} x_{i_0} \summ{i_1=1}{m}\summ{i_1=2}{m}\dotsm \summ{i_{\ell}=1}{m} \left(\prod_{k=1}^{\ell+1}  \1(\boldeta_{k,i_k}(\x)\geq0)  \right) \left(\prod_{k=1}^{\ell+1} u_{k, i_{k}, i_{k-1}}  \right)
\end{align*}


\end{proof}

\begin{proof}[Proof of Theorem 2]
Using Lemma 3 for $\ell={L-1}$ and some $i_{L-1} \in [m]$ we have
\[
h_{L-1, i_{L-1}}(\x) = \summ{i_0 = 1}{d} x_{i_0} \summ{i_1=1}{m}\summ{i_1=2}{m}\dotsm \summ{i_{L-2}=1}{m} \left(\prod_{k=1}^{L-1}  \1(\boldeta_{k,i_k}(\x)\geq0)  \right) \left(\prod_{k=1}^{L-1} U_{k, i_{k}, i_{k-1}}  \right)
\]

Consider the RHS of Equation \ref{eqn:to_prove}
\begin{align*}
U_L \left(\h_{L-1}(\x) \right)
&=    \sum_{i_{L-1}=1}^m U_{L,1,i_{L-1}} h_{L-1, i_{L-1}}(\x) \\
&= \sum_{i_{L-1}=1}^m U_{L,1,i_{L-1}} \summ{i_0 = 1}{d} x_{i_0} \summ{i_1=1}{m}\summ{i_1=2}{m}\dotsm \summ{i_{L-2}=1}{m} \left(\prod_{k=1}^{L-1}  \1(\boldeta_{k,i_k}(\x)\geq0)  \right) \left(\prod_{k=1}^{L-1} U_{k, i_{k}, i_{k-1}}  \right) \\
&= U_{L,1,i_{L-1}}  \summ{i_0 = 1}{d} x_{i_0} \summ{i_1=1}{m}\summ{i_1=2}{m}\dotsm \summ{i_{L-2}=1}{m} \sum_{i_{L-1}=1}^m\left(\prod_{k=1}^{L-1}  \1(\boldeta_{k,i_k}(\x)\geq0)  \right) \left(\prod_{k=1}^{L-1} U_{k, i_{k}, i_{k-1}}   \right) \\
&= \summ{i_0 = 1}{d} x_{i_0} \sum_{\pi \in [m]^{L-1}} f_\pi(\x)  \left(\prod_{k=2}^{L-1} U_{k, i_{k}, i_{k-1}}   \right) 
 U_{L,1,i_{L-1}} U_{1,i_1,i_0} \\
&=  \sum_{\pi \in [m]^{L-1}} f_\pi(\x)  \left(\prod_{k=2}^{L-1} U_{k, i_{k}, i_{k-1}} \right)  U_{L,1,i_{L-1}}  \summ{i_0 = 1}{d} x_{i_0} U_{1,i_1,i_0}  \\
&=  \sum_{\pi \in [m]^{L-1}} f_\pi(\x)  \left(\prod_{k=2}^{L-1} U_{k, i_{k}, i_{k-1}} \right)  U_{L,1,i_{L-1}}  \bu_{1,i_1}^\top \x \\
&=  \sum_{\pi \in [m]^{L-1}} f_\pi(\x)  g_\pi(\x) \\
\end{align*}

\end{proof}

\subsection{Proof of  Theorem \ref{thm:relu-theorem}}

\begin{theorem*}

Let $W_2, \ldots, W_{L-1}$ be matrices of size $m \times m$. Let $W_1 \in \R^{m\times d}$ and $W_L \in \R^{1\times m}$. For any path $\pi = (i_1, \ldots, i_{L-1}) \in [m]^{L-1}$, let the gating and expert models $f_\pi, g_\pi:\R^d \> \R$ be
\begin{align*}
f_\pi(\x) &= \prod_{\ell=1}^{L-1} \1\left(\w_{\ell,i_\ell}^\top \h_{\ell-1}(\x) \geq 0 \right)\\ 
g_\pi(\x) &= \left[ \prod_{\ell=2}^{L-1} w_{\ell,i_\ell,i_{\ell-1}}\right] w_{L,1,i_{L-1}} \w_{1,i_1}^\top \x
\end{align*}
Also, let $ \h$ be collections of mappings such that
\begin{align}
    \h_0(\x) &= \x \nonumber \\ 
    \text{ for } \ell \in \{1,\ldots,L-1\} \qquad  \h_\ell(\x) &= \1(W_\ell \h_{\ell-1}(\x) \geq 0) \circ \left( W_\ell \h_{\ell-1}(\x) \right) \label{eqn:h-vector-defn-relu}.
\end{align}
Then we have that
\begin{equation}
        \sum_{\pi \in [m]^{L-1}} f_\pi(\x) g_\pi(\x)  = W_L \left(\h_{L-1}(\x) \right)
        \label{eqn:to_prove_relu}
\end{equation}
\end{theorem*}


We prove an intermediate lemma based on induction first.

\begin{lem}
Let $\ell \in \{1,2,\ldots, L-1\}$, and let $i_\ell \in [m]$
\begin{equation}
\h_{\ell,i_\ell}(\x) = \summ{i_0 = 1}{d} x_{i_0} \summ{i_1=1}{m}\summ{i_1=2}{m}\dotsm \summ{i_{\ell-1}=1}{m} \left(\prod_{k=1}^{\ell}  \1(W_{k,i_k}\h_{k-1}(\x)\geq0)  \right) \left(\prod_{k=1}^\ell w_{k, i_{k}, i_{k-1}}  \right)
\label{eqn:lem-induction-relu}
\end{equation}
\end{lem}
\begin{proof}

We will prove by Induction.\\
Basis Step: When $\ell = 1$, the RHS of Equation \ref{eqn:lem-induction-relu} becomes
\[
\h_{1,i}(\x) = W_{1,i}\x \1(W_{1,i}\h_{0}(\x)\geq0) = W_{1,i}\x\1(W_{1,i}x\geq0)
\]
which is true by the definition of $\h$ from Equation \ref{eqn:h-vector-defn-relu} and that $\h_0(\x)=\x$.

Assume Equation \ref{eqn:lem-induction-relu} is true for some $\ell$. We now prove that it holds for $\ell+1$

From the definition of $\h$ in Equation \ref{eqn:h-vector-defn-relu} and the induction assumption, for any $i_{\ell+1} \in [m]$
\begin{align*}
\h_{\ell+1,i_{\ell+1}}(\x) 
&=  \1(W_{\ell+1,i_{l+1}} \h_{\ell}(\x) \geq 0) \left( W_{\ell+1, i_{l+1}} \h_{\ell}(\x) \right) \\
&=  \1(W_{\ell+1,i_{\ell+1}}\h_{\ell}(\x) \geq 0)  \sum_{i_\ell =1}^m w_{\ell+1,i_{\ell+1},i_\ell} h_{\ell, i_\ell}(\x)   \\
&= \summ{i_0 = 1}{d} x_{i_0} \summ{i_1=1}{m}\summ{i_1=2}{m}\dotsm \summ{i_{\ell}=1}{m} \left(\prod_{k=1}^{\ell+1}  \1(W_{k,i_{k}}\h_{k-1}(\x) \geq 0)  \right) \left(\prod_{k=1}^{\ell+1} w_{k, i_{k}, i_{k-1}}  \right)
\end{align*}
\end{proof}

\begin{proof}[Proof of Theorem 1]
Using Lemma 4 for $\ell={L-1}$ and some $i_{L-1} \in [m]$ we have
\[
h_{L-1, i_{L-1}}(\x) = \summ{i_0 = 1}{d} x_{i_0} \summ{i_1=1}{m}\summ{i_1=2}{m}\dotsm \summ{i_{L-2}=1}{m} \left(\prod_{k=1}^{L-1}  \1(W_{k,i_{k}}\h_{k-1}(\x) \geq 0)  \right) \left(\prod_{k=1}^{L-1} W_{k, i_{k}, i_{k-1}}  \right)
\]

Consider the RHS of Equation \ref{eqn:to_prove_relu}
\begin{align*}
W_L \left(\h_{L-1}(\x) \right)
&=    \sum_{i_{L-1}=1}^m W_{L,1,i_{L-1}} h_{L-1, i_{L-1}}(\x) \\
&= \sum_{i_{L-1}=1}^m W_{L,1,i_{L-1}} \summ{i_0 = 1}{d} x_{i_0} \summ{i_1=1}{m}\summ{i_1=2}{m}\dotsm \summ{i_{L-2}=1}{m} \left(\prod_{k=1}^{L-1}  \1(W_{k,i_{k}}\h_{k-1}(\x) \geq 0)  \right) \left(\prod_{k=1}^{L-1} W_{k, i_{k}, i_{k-1}}  \right) \\
&=  W_{L,1,i_{L-1}} \summ{i_0 = 1}{d} x_{i_0} \summ{i_1=1}{m}\summ{i_1=2}{m}\dotsm \summ{i_{L-2}=1}{m} \sum_{i_{L-1}=1}^m\left(\prod_{k=1}^{L-1}  \1(W_{k,i_{k}}\h_{k-1}(\x) \geq 0)  \right) \left(\prod_{k=1}^{L-1} W_{k, i_{k}, i_{k-1}}   \right) \\
&= \summ{i_0 = 1}{d} x_{i_0} \sum_{\pi \in [m]^{L-1}} f_\pi(\x)  \left(\prod_{k=2}^{L-1} W_{k, i_{k}, i_{k-1}}   \right) 
 W_{L,1,i_{L-1}} W_{1,i_1,i_0} \\
&=  \sum_{\pi \in [m]^{L-1}} f_\pi(\x)  \left(\prod_{k=2}^{L-1} W_{k, i_{k}, i_{k-1}} \right)  W_{L,1,i_{L-1}}  \summ{i_0 = 1}{d} x_{i_0} W_{1,i_1,i_0}  \\
&=  \sum_{\pi \in [m]^{L-1}} f_\pi(\x)  \left(\prod_{k=2}^{L-1} W_{k, i_{k}, i_{k-1}} \right)  W_{L,1,i_{L-1}}  w_{1,i_1}^\top \x \\
&=  \sum_{\pi \in [m]^{L-1}} f_\pi(\x)  g_\pi(\x) \\
\end{align*}

\end{proof}
\vfill



\end{document}